\documentclass[11pt,letterpaper]{berkeley}

\usepackage[all]{hypcap}

\PassOptionsToPackage{comma,numbers,sort,compress}{natbib}

\usepackage[comma,numbers,sort,compress]{natbib}
\usepackage{hyperref}[citecolor=lightblue]

\hypersetup{
    colorlinks = true,
    citecolor = {magenta},
}

\usepackage{algorithm}
\usepackage{algorithmicx}
\usepackage{algpseudocode}
\usepackage{microtype}
\usepackage{graphicx}
\usepackage{subcaption}
\expandafter\def\csname ver@subfig.sty\endcsname{}
\usepackage{booktabs} %
\usepackage{float}

\usepackage{amsmath}
\usepackage{amssymb}
\usepackage{mathtools}
\usepackage{amsthm}
\usepackage{mathrsfs}
\usepackage{nicefrac}
\usepackage{dsfont}
\usepackage{enumitem}
\usepackage{subcaption}
\usepackage{graphicx,subfig}

\usepackage{float}

\usepackage[utf8]{inputenc} %
\usepackage[T1]{fontenc}    %
\usepackage{hyperref}       %
\usepackage{url}            %
\usepackage{booktabs}       %
\usepackage{amsfonts}       %
\usepackage{nicefrac}       %
\usepackage{microtype}      %
\usepackage{xcolor}         %
\usepackage{graphicx}
\usepackage{subcaption}
\usepackage{amssymb}
\usepackage{fdsymbol}
\usepackage{wrapfig}
\usepackage{lipsum}
\usepackage{enumitem}
\usepackage{stackengine}
\usepackage[font=small,labelfont=bf]{caption}
\usepackage{color}
\usepackage{adjustbox}

\newcommand{\eg}{\emph{e.g.},\ }
\newcommand{\ie}{\emph{i.e.},\ }

\newcommand{\by}{\mathbf{y}}

\newcommand{\bx}{\mathbf{x}}

\definecolor{blanchedalmond}{rgb}{1.0, 0.92, 0.8}
\definecolor{carmine}{rgb}{0.59, 0.0, 0.09}
\definecolor{lightblue}{rgb}{0.22,0.45,0.70}%

\newtheorem{theorem}{Theorem}[section]

\renewcommand{\mathbf}{\boldsymbol}

\makeatletter
\def\Ddots{\mathinner{\mkern1mu\raise\p@
\vbox{\kern7\p@\hbox{.}}\mkern2mu
\raise4\p@\hbox{.}\mkern2mu\raise7\p@\hbox{.}\mkern1mu}}
\makeatother

\definecolor{amaranth}{rgb}{0.9, 0.17, 0.31}
\definecolor{antiquebrass}{rgb}{0.8, 0.58, 0.46}
\definecolor{antiquefuchsia}{rgb}{0.57, 0.36, 0.51}
\definecolor{chromeyellow}{rgb}{0.31, 0.47, 0.26}
\newcommand{\sft}{\mathrm{sft}}
\newcommand{\negtext}{\text{neg}}

\newcommand{\hatby}{\hat{\by}}

\newcommand{\rftdata}{\mathcal{D}^+_{\pi_\sft}}
\newcommand{\dpodata}{\mathcal{D}^{\pm}_{\pi_\sft}}
\newcommand{\syndata}{\mathcal{D}_{\mathrm{syn}}}

\newcommand{\posby}{+\by}
\newcommand{\negby}{-\by}
\newcommand{\hatposby}{+\hat{\by}}
\newcommand{\hatnegby}{-\hat{\by}}

\newcommand{\bok}{\text{BoK}}

\newcommand{\brck}[1]{\left [ #1 \right ] }

\usepackage[most,skins,theorems]{tcolorbox}

\tcbset{
  aibox/.style={
    width=\linewidth,
    top=8pt,
    bottom=4pt,
    colback=blue!6!white,
    colframe=black,
    colbacktitle=black,
    enhanced,
    center,
    attach boxed title to top left={yshift=-0.1in,xshift=0.15in},
    boxed title style={boxrule=0pt,colframe=white,},
  }
}
\newtcolorbox{AIbox}[2][]{aibox,title=#2,#1}
\definecolor{lightblue}{rgb}{0.22,0.45,0.70}%

\usepackage{amsmath}

\usepackage[all]{hypcap}

\title{RL on Incorrect Synthetic Data Scales the Efficiency of LLM Math Reasoning by Eight-Fold}

\author[1]{Amrith Setlur}
\author[1]{Saurabh Garg}
\author[2]{Xinyang (Young) Geng}
\author[3]{Naman Garg}
\author[1]{Virginia Smith}
\author[2]{Aviral Kumar}

\affil[1]{Carnegie Mellon University}
\affil[2]{Google DeepMind}
\affil[3]{MultiOn}

\correspondingauthor{Amrith Setlur; asetlur@andrew.cmu.edu. This work was done entirely at Carnegie Mellon University (CMU). \\Data and code will be released here: \href{https://github.com/ars22/scaling-LLM-math-synthetic-data}{https://github.com/ars22/scaling-LLM-math-synthetic-data}.}

\begin{abstract}
Training on model-generated synthetic data is a promising approach for finetuning LLMs, but it remains unclear when it helps or hurts. In this paper, we investigate this question for math reasoning via an empirical study, followed by building a conceptual understanding of our observations. 
First, we find that while the typical approach of finetuning a model on synthetic correct or \emph{positive} problem-solution pairs generated by capable models offers modest performance gains, sampling more correct solutions from the finetuned learner itself followed by subsequent fine-tuning on this self-generated data \textbf{doubles} the efficiency of the same synthetic problems. At the same time, training on model-generated positives can amplify various spurious  correlations, resulting in flat or even inverse scaling trends as the amount of data increases. Surprisingly, we find that several of these issues can be addressed if we also utilize \emph{negative} responses, \ie model-generated responses that are deemed incorrect by a final answer verifier. Crucially, these negatives must be constructed such that the training can appropriately recover the utility or advantage of each intermediate step in the negative response. With this \emph{per-step} scheme, we are able to attain consistent gains over only positive data, attaining performance similar to amplifying the amount of synthetic data by $\mathbf{8 \times}$. We show that training on per-step negatives can help to unlearn spurious correlations in the positive data, and is equivalent to advantage-weighted reinforcement learning (RL),
implying that it inherits robustness benefits of RL over imitating positive data alone. 
\end{abstract}

\begin{document}

\maketitle

\vspace{-0.4cm}
\section{Introduction}
\label{sec:intro}
\vspace{-0.15cm}

Training large language models (LLMs) relies on the ability to train on large amounts of high-quality data. It is predicted that we will run out of high-quality internet data by 2026~\cite{villalobos2022will,liu2024best}, necessitating training on model-generated data, or what is commonly referred to as \emph{synthetic data}. Recent trends illustrate that scaling up synthetic data can lead to improvements~\citep{li2024common,chen2024self} on hard reasoning problems, while other results illustrate that training on synthetic data can steer the performance of the model into a downward spiral~\citep{shumailov2023curse,alemohammad2023self,gerstgrasser2024model}---amplying biases, misinformation, and undesired stylistic properties. Thus while \emph{in principle}, synthetic data could potentially address data scarcity, it must be designed in an appropriate manner to be effective. However, this has been hard due to a lack of an understanding of how synthetic data contributes to LLM behavior.

To provide clarity on how synthetic data contributes to performance, we aim to understand its impact on LLM capabilities via a study on math reasoning, a prevalent scenario where synthetic data is used. Typically, in this  setting, synthetic data corresponds to correct or \emph{positive} model-generated responses for a novel set of initial problems synthesized by prompting capable models~\citep{li2024common,liu2023exploration}. 
The resulting model is then evaluated on a held-out set of problems drawn from a test set. Perhaps as expected, we find that performance improves when finetuning models on positive synthetic responses, though the scaling rates for performance improvement are often substantially slower than those observed during pretraining. Concretely, we find that under the scaling law of \citet{zhang2024scaling}, the error rate scales as $\approx$${D}^{-0.05}$ to ${D}^{-0.15}$ in the size $D$ of synthetic dataset. Second, we observe that not all types of positive synthetic data are equally effective: often positive responses self-generated by the learner itself are as effective as $2\times$ synthetic data from bigger models in improving performance. This is because responses from a similar model are ``easier-to-fit'' than those from a more capable model, resulting in reduced memorization~\citep{kang2024unfamiliar,tirumala2022memorization} during finetuning. We also observe that if the positive response contains incorrect/irrelevant intermediate steps, training on such data often incentivizes the model to overfit on spurious correlations, leading to a flat or even inverse scaling with more data.

\begin{figure*}
    \vspace{-0.37cm}
    \captionsetup{font=footnotesize,skip=6pt}
  \centering
  \includegraphics[width=0.74\linewidth]{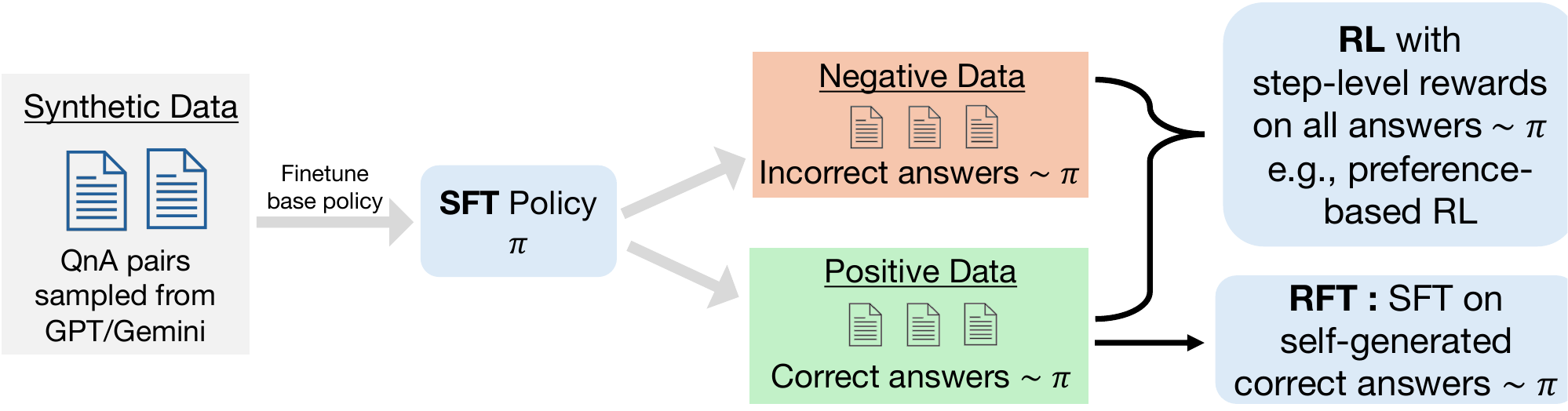}
  ~\vline~
  \includegraphics[width=0.21\linewidth]{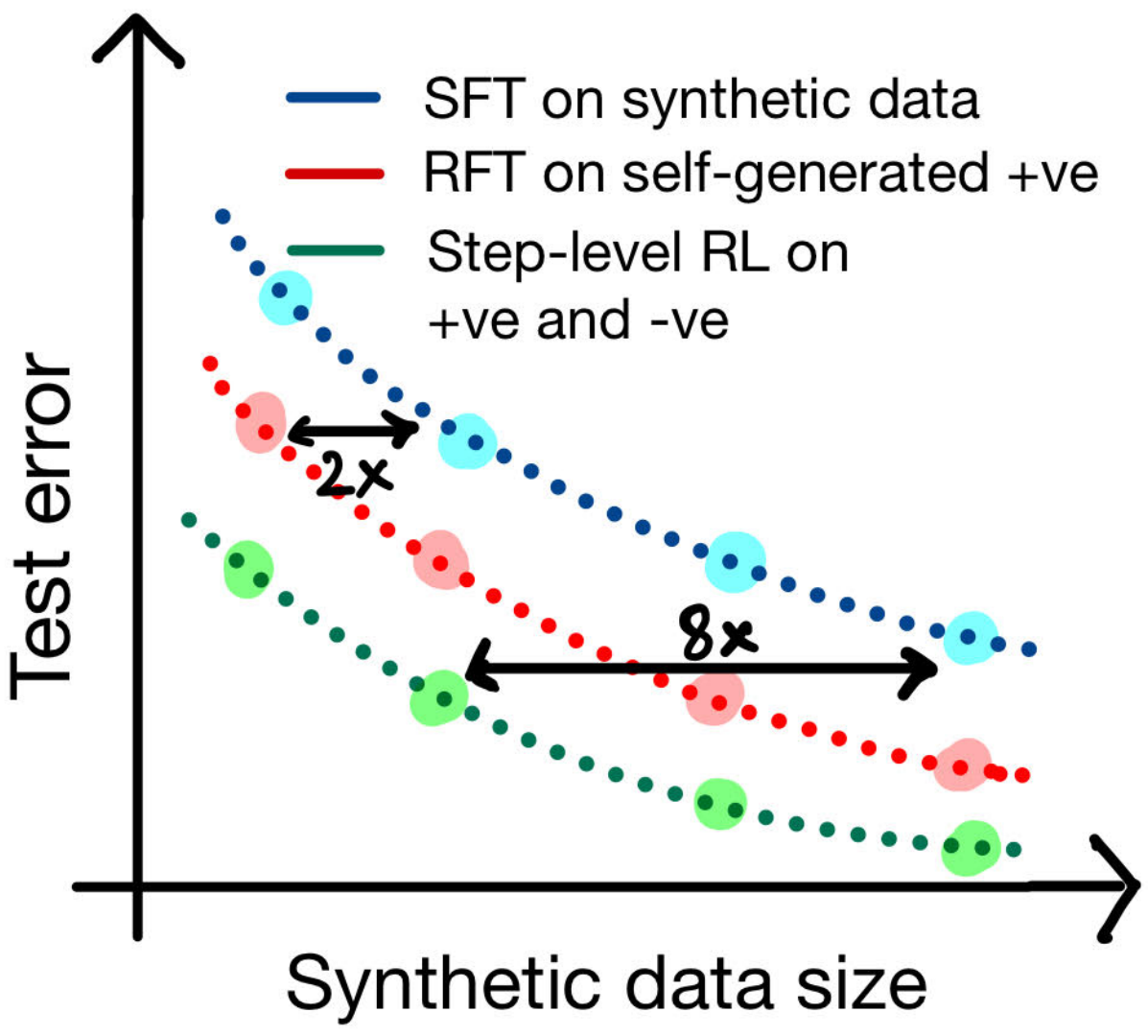}
  
  \caption{
  \label{fig:intro-figure}
  \textbf{\emph{Left: Positive/negative synthetic data:}}  Pictorial representation of our synthetic data pipeline, definitions for positive and negative synthetic data and how they are fed to supervised finetuning (SFT; positive synthetic data from more capable models), rejection finetuning (RFT; positive self-generated synthetic data from the SFT model) and step-level RL (via per-step DPO) algorithms we analyze in our work. \emph{\textbf{Right: An illustration of our main result:}} using RFT on self-generated positives scales the efficiency of SFT on synthetic problems and responses by 2x, whereas using step-level RL with negative data scales the efficiency by 8x.
  }
  \vspace{-0.3cm}
\end{figure*}
Perhaps surprisingly, we find that the aforementioned pathologies of training on positive data only can be addressed if we also utilize synthetic \emph{negative} responses: responses generated by the model that do not result in obtaining a correct final answer. One way to utilize negative responses is via methods such as  direct preference optimization (DPO)~\citep{rafailov2023direct}. While performance of standard DPO~\citep{rafailov2023direct} largely flatlines as the number of synthetic problems are scaled up (Figure~\ref{fig:neg-scaling}), we are able to attain consistent improvements if the negative data is generated appropriately. A solution trace for a math problem typically comprises of multiple \emph{reasoning steps} corresponding to intermediate results.  Our insight is that instead of contrasting arbitrary correct and incorrect responses, we should contrast those positive and negative responses that depict good and bad choices for the more ``critical'' intermediate steps: steps that the model must carefully produce so as to succeed at the problem. In other words, critical steps are those which the model is unable to recover from, and hence, must be emphasized. %
With this scheme, we are able to attain consistent gains over only positive data, \textbf{attaining performance similar to scaling up positive synthetic data by $8 \times$.} We also show that training on this sort of negative data evades spurious steps amplified by training on positive data alone.

To theoretically understand our findings, we build a conceptual model of how training on this data benefits performance. Formally, we show that this construction of negative data, which emphasizes ``critical'' tokens (Figure~\ref{fig:critical_states}) enables us to perform credit assignment, and is equivalent to training the model with per-step advantage-weighted reinforcement learning (RL)~\citep{peng2019advantage} on a mixture of positive and negative synthetic data. Specifically, these advantage values are computed under an optimal value function induced by sampling multiple responses under the SFT policy obtained by training on only the positive data. This reduction of using negative data to advantage-weighted RL enables us to conceptually compare it to training on positive data, which corresponds to imitation learning (\ie behavioral cloning) on positive data. 
First, we are able to argue for the generalization gains of advantage-weighted RL through the lens of distribution robust objectives.
Second, building on theoretical results in RL~\citep{kumar2022should}, we are also able to show that when advantages can be estimated reliably, advantage-weighted RL will be significantly more sample-efficient compared to imitation. 
Overall, this model explains the utility of negative data over only positive data.

Our contribution is a study of the role of synthetic data in improving math reasoning capabilities of LLMs. We derive scaling laws for positive and negative data on common reasoning benchmarks such as GSM8K~\citep{cobbe2021gsm8k} and MATH~\citep{hendrycksmath2021}, and observe that: \textbf{(a)} training on positive synthetic data from capable models results in scaling rates that are significantly slower than standard empirical risk minimization; \textbf{(b)} training on model-generated positive synthetic data can improve sample efficiency  by $2\times$ but also amplifies spurious correlations; \textbf{(c)} appropriate ways of constructing learner-specific negative data with emphasis on critical steps,  results in a performance boost equivalent to scaling up positive data $8\times$; \textbf{(d)} training with negative data provides a mechanism to unlearn spurious correlations; and \textbf{(e)} we present a conceptual model inspired from RL to explain our observations for synthetic data, and use it to explain the generalization benefits we see.

\vspace{-0.35cm}
\section{Related Work}
\label{sec:relwork}
\vspace{-0.2cm}

A standard procedure to finetune a pretrained LLM is teacher-forcing on  expert data, \ie maximizing the likelihood of the next token given all previous tokens~\cite{williams1989learning,brown2020language}. In Appendix~\ref{sec:additional-relwork} we discuss some failure modes of this procedure for math reasoning that positive or negative synthetic data can address. 

\textbf{Positive synthetic data.}
Learning theory dictates that the SFT policy trained on more SFT data (\eg 1.5M for DeepSeek-Math~\cite{bi2024deepseek}) would have improved math reasoning capbabilities.
Thus, a common goal for generating synthetic data as close as possible to the SFT data~\cite{li2024common,liu2023exploration,liu2024best}. That said, generating high quality math data can be challenging, since verification can often be hard.
When synthetic data is verified by larger models~\cite{sharma2024critical,wang2024mathshepherd}, recent works~\cite{luo2023wizardmath,yu2024metamath} observe scaling similar to finetuning LLMs on expert data~\cite{zhang2024scaling,yuan2023scaling}, while another work~\cite{dong2023abilities} notes the compositional gains from SFT data for code generation.
Common sources of ``good'' synthetic data include responses from stronger teachers~\cite{li2024common,lightman2023lets},
 or data generated by the SFT policy 
 itself, in the framework of reinforced self-training (ReST) and STaR ~\cite{zelikman2022star,singh2024human,chen2024self,yuan2023scaling}. 
In our work, we study and compare the performance scaling with positive synthetic data from bigger models like GPT-4 and Gemini 1.5 Pro with self-generated positive data. We connect our findings to evidence showing ``ease of learning'' generalizable features on self-generated completions~\cite{kang2024unfamiliar} which often prevents undesirable memorization~\cite{tirumala2022memorization}. 
Finally, our work also sheds light on several concerns about training on synthetic positive data amplifying biases~\cite{seddik2024bad,wyllie2024fairness}, and leading to model collapse~\cite{dohmatob2024model,gerstgrasser2024model}, especially due to overfitting on``spurious'' intermediate steps. We conceptually explain this phenomenon and also discuss how negative model-generated responses can help identify and unlearn those spurious steps.

\textbf{Benefits and nuances of negative synthetic data.} 
While most works on synthetic data for math reasoning~\citep{yu2024metamath,li2024common,liu2024best,yuan2023scaling} focus on training on positive (correct) answers, our work also studies complementary gains from negative (incorrect) completions generated by the SFT policy~\cite{hwang2024self,pal2024smaug,yuan2024self,pang2024iterative}.
To leverage sub-optimal negative data, we adopt the generic framework of offline preference  optimization~\cite{rafailov2023direct,ethayarajh2024kto,zhao2022calibrating}, where a preference pair is constructed using correct and incorrect responses for the same problem~\cite{pal2024smaug}. 
Despite numerous studies on preference data composition~\citep{chen2024self,cheng2023adversarial,tajwar2024preference,vicuna2023,wang2023rlhf,munos2023nash,swamy2024minimaximalist}, it remains unclear what is the best approach to pose a reasoning problem as a preference optimization problem. Randomly pairing correct and incorrect completions in a preference pair can lead to poor performance~\citep{pang2024iterative,hong2024reference,xu2024contrastive,pal2024smaug} due to objective mismatch~\citep{tajwar2024preference,zhang2024negative} and requires auxilliary losses to perform well. Another option is to utilize negative data for training verifiers~\cite{hosseini2024v,yu2023outcome} but this line of work still only trains the policy using positive data.
We introduce a conceptual model of negative data, where we understand how certain choices of negative data can assign per-step credits, which we then use to establish the equivalence of preference optimiztion to to advantage weighted RL. Self-explore method in \citet{hwang2024self} can be viewed as an special instance of our general framework. Another work exploiting per-step credit assignment is \citet{wang2024mathshepherd}: through tree-based sampling they identify and use the reasoning subsequence that led to the most incorrect answers under the SFT policy for training a reward model. While this is indeed related, our conceptual model and analysis also aims to understand why assigning per-step credits  can generalize better by unlearning spurious correlations, \textit{e.g.,} when the credits are given by the Q-function of the ``best-of-K'' SFT policy.

\vspace{-0.25cm}
\section{Problem Setup and Synthetic Data Generation Pipeline}
\label{sec:preliminaries}
\vspace{-0.25cm}

Our goal in this paper is to understand the role of synthetic data in producing strong language model reasoners. Building on the recipe of \citet{li2024common,liu2023exploration}, we collect synthetic data consisting of both novel problems designed by capable models such as GPT4~\citep{openai2024gpt4} and Gemini 1.5 Pro~\citep{geminiteam2024gemini}, and responses to these problems, obtained from the same models. Concretely, we focus on two mathematical reasoning benchmarks: GSM8K~\citep{cobbe2021gsm8k} and MATH~\citep{hendrycksmath2021}. 

\textbf{Synthetic data pipeline.} First, given a dataset $\mathcal{D}_\text{real} = \{(\bx^r_i, \by^r_i)\}$ of problems $\bx_i^r \sim p_\text{real}(\bx)$ and solution traces $\by_i^r \sim p_\text{real}(\by\mid\bx_i)$, we prompt one of the highly-capable models with a uniformly random sample $(\bx^{r}_i, \by^{r}_i) \in \mathcal{D}_\text{real}$ and ask the model to generate a new problem $\bx_i$ such that it is similar to the real problem $\bx_i^{r}$, in a way that a feasible solution exists.
Second, we ask the model to provide a solution trace answer $\by_i$ with step-by-step reasoning (exact prompts for $\bx_i, \by_i$ are borrowed from \citet{li2024common}, shown in Appendix~\ref{sec:synth_data_appendix}). 
We assume that the answers generated via this process are accurate, and perform lightweight filtering step to remove duplicates, badly-formatted answer traces, and model failures. 
Based on the above, for any synthetic problem and solution pair $(\bx, \by)$, we can define a binary reward function 
$r(\by, \hat{\by}) \mapsto \{0,1\}$, 
which verifies if a new solution trace $\hat{\by}$ is correct or not. This is implemented with a set of answer extraction and string matching tools borrowed from~\cite{yu2024metamath,li2024common}.
We say that a new trace $\hat{\by}$ 
is a {\emph{positive}} trace if it produces the correct final answer 
\ie $r(\hat{\by}, \by)=1$, and {\emph{negative}} if it produces an incorrect final answer, \ie $r(\hat{\by}, \by)=0$. 
By definition, $r(\by, \by)=1$, and the original trace $\by$ is always positive.

\textbf{Positive and negative datasets.}  The above process induces a joint distribution $p_\text{syn}(\bx, \by)$, \textit{iid} samples from which yields positive synthetic dataset  $\syndata$. 
We note that the sampling process for $\syndata$ is designed to ensure that the induced marginal distribution over synthetic problems  $p_\text{syn}(\bx)$ is close to $p_\text{real}(\bx)$. 
We will use $\mathcal{D}_\pi^{+}$ to denote the positive dataset of $(\bx, \hatposby)$ where $\hatposby$ is a positive solution trace generated from some policy $\pi(\cdot \mid \bx)$. 
For a positive $\hatposby$ and negative $\hatnegby$ trace, sampled from the same policy $\pi(\cdot \mid \bx)$, we denote a dataset over problems and solution pairs: $(\bx, \hatposby, \hatnegby)$ as $\mathcal{D}^{\pm}_\pi$.

\textbf{Reasoning steps.} The trace $\by_i$ consists of several intermediate steps, $\by_i = \brck{\by_{i,1}, \ldots, \by_{i,L}}$. We assume each trace has at most $L$ steps, and use $\by_{1:t}$ to denote the subsequence of first $t$ steps. Since mathematical reasoning problems require step-by-step computation, simply arriving at an incorrect final answer does not mean that all  steps in a negative $\hat{\by}$ are incorrect. 
Similarly, a positive $\hat{\by}$ may also have incorrect reasoning steps. 
In fact, even the original answers generated by more capable models in $\syndata$ may also contain incorrect reasoning steps, and training on such  traces may actually lead to unintended consequences (Section~\ref{sec:positive-synth-expts}).

\vspace{-0.3cm}
\section{Learning from Synthetic Data}
\label{sec:learning}
\vspace{-0.15cm}
In this section, we discuss various algorithms for learning from the synthetic dataset $\syndata$ discussed in the previous section, as well as positive and negative solution traces generated using a model. 

\textbf{Supervised and rejection finetuning (SFT and RFT).} Given positive  synthetic $\syndata$, perhaps the most straightforward approach (and the most prevalent) is to learn $\pi_\sft$ on this data via supervised next-token prediction: $\pi_\mathrm{sft} (\cdot|\bx) := \arg \max_{\pi} \mathbb{E}_{\bx, \by \sim \syndata} \left[ \log \pi(\by|\bx) \right]$. Another option is to train via supervised next-token prediction on problems in $\syndata$, but when using a positive solution trace $\hat{\by}$ sampled from $\pi_\text{sft}(\cdot|\bx)$, instead of positive synthetic responses from the capable models in $\syndata$. 
Akin to rejection finetuning (RFT~\citep{yuan2023scaling}) or STaR~\citep{zelikman2022star}, 
sampling from $\pi_\text{sft}(\cdot\mid\bx)$ once is not guaranteed to give a positive response, and we instead sample $M$ times for each $\bx$ and construct the dataset $\rftdata$ of SFT policy generated positive responses. 
Then, we apply the next-token prediction loss on $\rftdata$.

\textbf{Preference optimization.} Beyond positive data, we can also learn from negative synthetic data generated from the SFT policy, especially when contrasted with positive responses. 
However, learning from negative data presents multiple open design questions pertaining to the construction of negative traces, and the choice of the loss function, and simple supervised fine-tuning will not be a good choice since it will incentivize the model to produce more errors.
Therefore, we utilize a contrastive training approach, direct preference optimization (DPO~\citep{rafailov2023direct}) for incorporating negative data from $\pi_\sft$. In a nutshell, DPO trains a policy using the following preference optimization objective:
\begin{align}
\label{eq:dpo_objective}
    \min_\pi~~ \mathcal{L}_\text{DPO}(\pi) := \mathbb{E}_{(\bx, \posby, \negby) \sim \mathcal{D}_{\pi_\sft}^{\pm}}  \left[ \sigma \left( \beta \log \frac{\pi(\posby\mid\bx)}{\pi_\sft(\posby\mid\bx)} - \beta \log \frac{\pi(\negby\mid\bx)}{\pi_\sft(\negby\mid\bx)} \right) \right].
\end{align}
We consider two objectives that construct negative data $\hatnegby$ in distinct ways and subsequently train the model on that data using Equation~\ref{eq:dpo_objective}. The first variant is \textbf{\emph{standard DPO}}~\cite{rafailov2023direct}, which samples negative data $\hatnegby$ from the SFT policy (with rejection sampling) and adds  $(\bx, \by, \hatnegby)$ to $\dpodata$. 
The second variant is \textbf{\emph{per-step DPO}}~\citep{hwang2024self}, which first samples a complete solution trace $\hatby_{1:L}$ from $\pi_\sft$ and then determines the ``first  pit'' $\hatby_c$. The first pit $\hatby_c$ is the step where any completion following the step: $\hat{\by}_{c+1:L} \sim \pi_\text{sft}(\cdot \mid \bx, \hatby_{1:c})$ leads to incorrect answers in expectation under $\pi_\sft$.  The triplet $(\bx,\by,\hatby_{1:c})$ is added to the preference  dataset  $\dpodata$.

\vspace{-0.2cm}
\section{Positive Data Improves Coverage, But  Amplifies Spurious Correlations}
\label{sec:positive-synth-expts}
\vspace{-0.15cm}

We first analyze the influence of scaling up positive synthetic data on GSM8K and MATH. In this experiment, we fine-tune DeepSeek-Math-7B~\cite{bi2024deepseek} and LLama2-7B~\cite{touvron2023llama} models (details in Appendix~\ref{sec:implementation}) on varying sizes of $\syndata$, constructed out of a 5:1 mixture of GPT-4-turbo~\cite{openai2024gpt4} and Gemini-1.5 Pro~\cite{geminiteam2024gemini}\footnote{
We use API versions of these models as of May 5, 2024. All experiments in this paper (including experiments with GPT and Gemini models) were performed entirely by Amrith Setlur and Saurabh Garg.
}. 
We obtain a series of SFT policies on this data scaling ladder. We then train a series of models by running one iteration of RFT on data obtained from the SFT policies at each step.

\begin{figure}[!t]
     \captionsetup{font=small,skip=0pt}
    \centering
    \begin{subfigure}[b]{0.33\textwidth}
    \captionsetup{font=small,skip=0pt}
        \centering
        \includegraphics[width=0.99\textwidth]{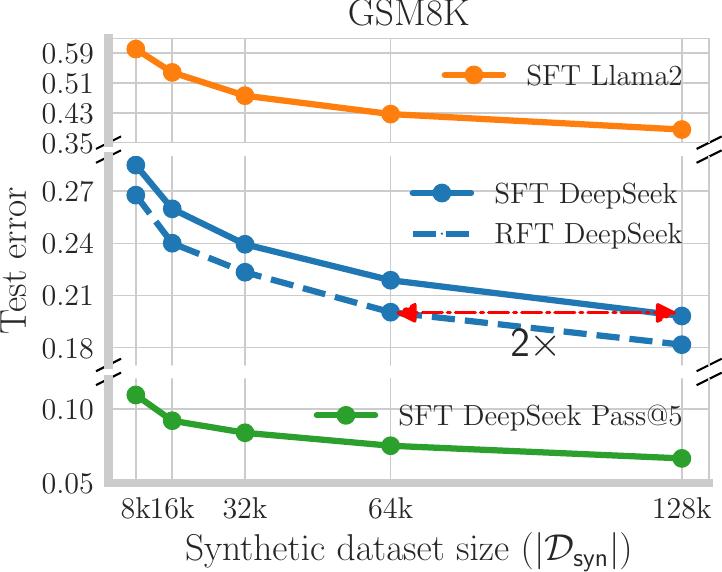}
        \caption*{(a)}
    \end{subfigure}\hfill
    \begin{subfigure}[b]{0.33\textwidth}
     \captionsetup{font=small,skip=0pt}
        \centering
        \includegraphics[width=0.99\textwidth]{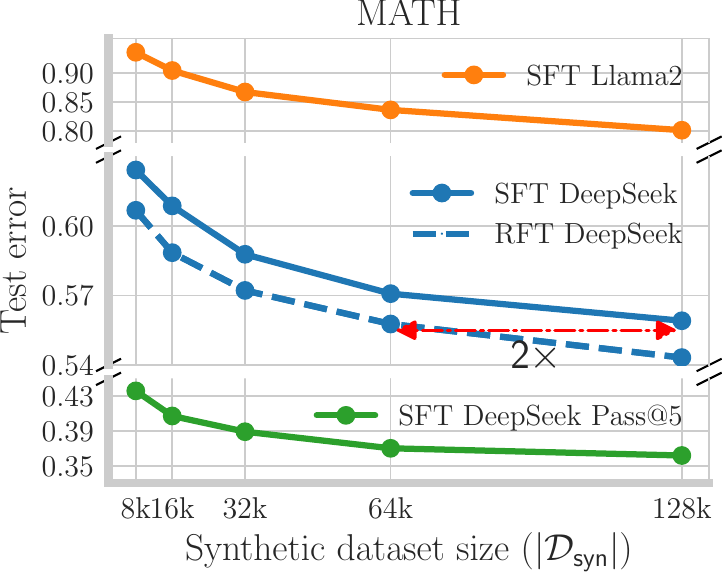}
        \caption*{(b)}
    \end{subfigure}\hfill
    \begin{subfigure}[b]{0.33\textwidth}
    \captionsetup{font=small,skip=0pt}
        \centering
        \includegraphics[width=0.99\textwidth]{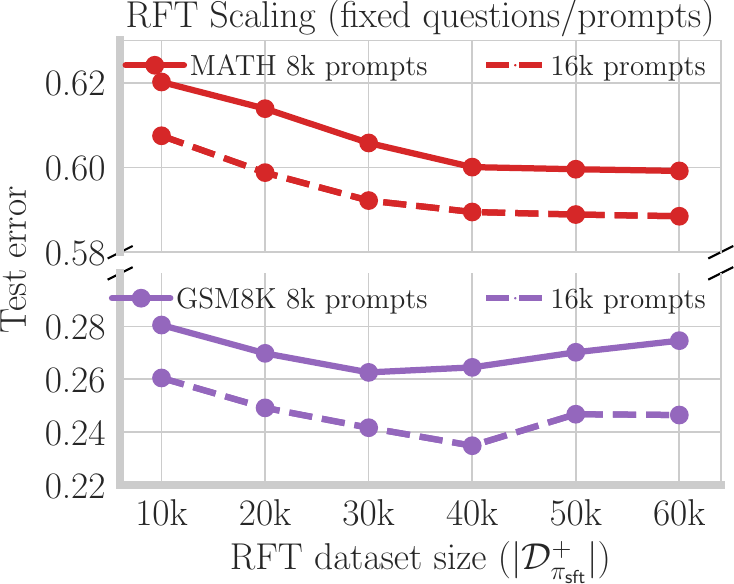}
        \caption*{(c)}
        \label{fig:sub2}
    \end{subfigure}
    \caption{\footnotesize{{\textbf{\emph{Positive data scaling laws:}} On GSM8K (a) and MATH (b), we evaluate SFT trained on $\syndata$ and RFT that uses SFT policy generated positives ($\rftdata$), as we scale $\syndata$, observing $\rftdata$ to be $2\times$ as effective as $\syndata$. In (c), we plot performance of RFT the number of correct solutions in $\rftdata$ are scaled, for a fixed set of 8k/16k problems from $\syndata$, observing that scaling model positives can amplify spurious correlations.}\vspace{-0.4cm}}}
    \label{fig:pos-scaling}
\end{figure}

{\textbf{Scaling results with positive synthetic data GPT-4 and Gemini-1.5 Pro.}} 
Since we assume that these more capable models generate correct solutions for new problems, by scaling $\syndata$ we are increasing \emph{coverage} under $p_\mathrm{real}$, \ie  adding new $\bx, \by$ with non-zero probability under $p_\mathrm{real}$. 
In Figures~\ref{fig:pos-scaling}(a,b), we plot the test error rate of the SFT policy as  $\syndata$ is scaled. As expected, we observe that the test error rate on both GSM8K and MATH improves with more positive data. 
Further, by simply fitting the parametric scaling law from~\cite{zhang2024scaling}, for $D:=|\syndata|$, we find that the scaling trends decay as $\approx$$D^{-0.15}$ on GSM8K and $\approx$$D^{-0.05}$ on the harder MATH dataset, with similar trends for the corresponding pass@$5$ error rates. 
Since these scaling trends are much more underwhelming than those for pre-training~\citep{hoffmann2022training}, this perhaps implies that samples in $\syndata$ are indeed improving coverage over samples in $p_\text{real}(\bx, \by)$, but maybe not as efficiently as sampling \textit{i.i.d.} samples directly from this ground-truth distribution.

{\textbf{Scaling results with positive synthetic data from 7B SFT policy.}} Previously, we scaled problems in $\syndata$ by querying GPT-4 and Gemini-1.5. 
Now, for existing problems in $\syndata$ 
we generate new responses by sampling from the $\pi_\sft$ trained on problems+solutions in $\syndata$.
For any  $(\bx, \by) \in \syndata$  we generate verified positive solution traces $\hat{\by} \sim \pi_\sft$ s.t. $r(\hat{\by}, \by)=1$. 
Following ~\citet{yuan2024advancing}, to ensure we sample enough correct responses, we sample $100$ times from $\pi_\sft$ and generate RFT datasets $\rftdata$, where each problem has atmost $4$ correct and diverse solutions. 
Note that while the number of questions in the SFT and RFT dataset remain the same, the RFT dataset contains more positive solutions per question. 
Since final answer verification is cheap, we are able to take advantage and have more self-generated solution traces in the RFT data mixture to improve coverage over responses. Note however that the number of synthetic problems and queries to highly capable models remain the same with both RFT and SFT. 
For completeness, we also study the setting when exactly one self-generated response per problem is provided in the RFT mixture and compare it to running SFT on data from more capable models later in this section.

We finetune the pretrained DeepSeek-Math-7B model on these new series of RFT datasets and plot the performance on GSM8K and MATH (Figure~\ref{fig:pos-scaling}(a,b)). First, \textcolor{lightblue}{\textbf{we observe that for any size of $\syndata$, the performance of the RFT model is better than the corresponding SFT model}}, and the difference remains consistent as we scale $\syndata$. 
Surprisingly, this indicates that training on multiple positive answer traces from the 7B $\pi_\sft(\by \mid \bx)$ can lead to better performing policies than training on a single trace from capable models. 

\textbf{What is the value of positives from $\pi_\sft(\by \mid \bx)$ compared to positives from ``oracle'' highly-capable models?} 
If sampling from $\pi_\sft$ also improves coverage and 
performance, then should we scale problems and solutions in $\syndata$, or just solutions in $\rftdata$? To answer this, we need to assess the importance of the RFT dataset $\rftdata$ in terms of $|\syndata|$. We do this by training SFT policies on $\syndata$ of sizes 8k and 16k, and then generating RFT datasets from the corresponding SFT policies where we only add more correct solution traces (for the same problems). In particular, for $\syndata$ with 8k and 16k questions, we scale the corresponding RFT datasets from 10k to 50k by adding positive solution traces from $\pi_\sft$ models trained on respective $\syndata$ where the solutions are from more capable models. For the RFT datasets constructed this way, we are fixing the number of questions and only scaling solutions,  unlike the RFT datasets in Figure~\ref{fig:pos-scaling}(a,b) where  both questions and answers scale. In Figure~\ref{fig:pos-scaling}(c) we plot the error rate of DeepSeek-Math-7B finetuned on the different sizes of $\rftdata$. 
Comparing the lowest values of the curves in Figure~\ref{fig:pos-scaling}(c) with $\syndata$ scaling in Figure~\ref{fig:pos-scaling}(a,b), we note that \textcolor{lightblue}{\textbf{performance from  $\rftdata$ is $2\times$ the size of $\syndata$ used to train $\pi_\sft$.}} 
 We also note that performance can plateau (or worsen in the case of GSM8K) as we scale up $\rftdata$ by a lot. 
 This is because $r(\cdot, \by)$ is unable to verify the correctness of each step a solution trace in $\rftdata$. We will show how these incorrect steps can induce spurious correlations that get amplified with further training, explaining this drop.

\begin{wrapfigure}{r}{0.38\textwidth}
    \vspace{-0.3cm}
    \captionsetup{font=footnotesize,skip=2pt}
  \centering
  \includegraphics[width=0.99\linewidth]{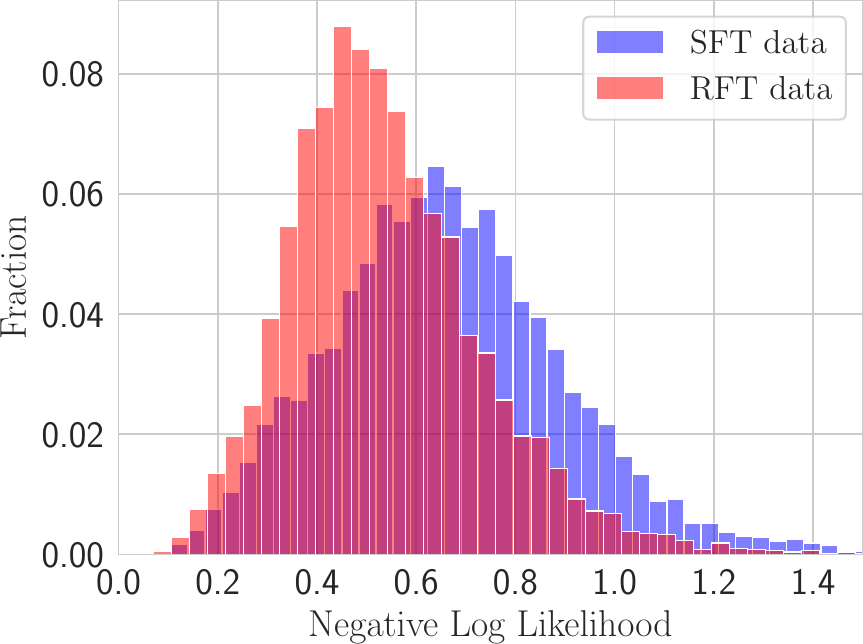}
  \caption{\label{fig:nothardtofit}\footnotesize{Under base LLM, RFT data sampled from $\pi_\sft$,  has higher likelihood than SFT data.}}
  \vspace{-0.3cm}
\end{wrapfigure}
\textbf{Why is self-generated positive data more sample-efficient?} 
From our result above, we find that solutions sampled from $\pi_\sft$ (trained on $\syndata$) yield better models, as good as those trained on $2\times |\syndata|$ problems.
This finding is surprising since one might expect more capable GPT-4/Gemini models to present better solutions (especially on problems generated from the same models), training on which should lead to good performance, akin to distillation~\citep{sharma2024critical}, but this is not the case.
Our results are consistent with the study of memorization  in 
LLMs~\citep{kang2024unfamiliar,hartmann2023sok, tirumala2022memorization}, which shows that pretrained (base) LLMs tend to memorize ``hard-to-fit'' and ``out-of-pretraining-distribution'' responses during finetuning, resulting in imperfect generalization. 
In contrast, correct response traces produced by $\pi_\sft$ on problems from $\syndata$ are not as hard-to-fit or as out-of-distribution, since they are obtained from a model that is ``close'' to the base LLM. We confirm this hypothesis 
with a histogram of negative log-likelihood values of the SFT and RFT data under the base LLM  (Figure~\ref{fig:nothardtofit}).
Hence, we expect STaR/RFT to alleviate the memorization problem on a large chunk of examples. 
This finding also corroborates \citet{yuan2023scaling}'s result that lower the perplexity of SFT data under the base model, the smaller the gap between SFT and RFT performance.

\begin{wrapfigure}{r}{0.4\textwidth}
    \vspace{-0.3cm}
    \captionsetup{font=footnotesize,skip=2pt}
  \centering
  \includegraphics[width=0.99\linewidth]{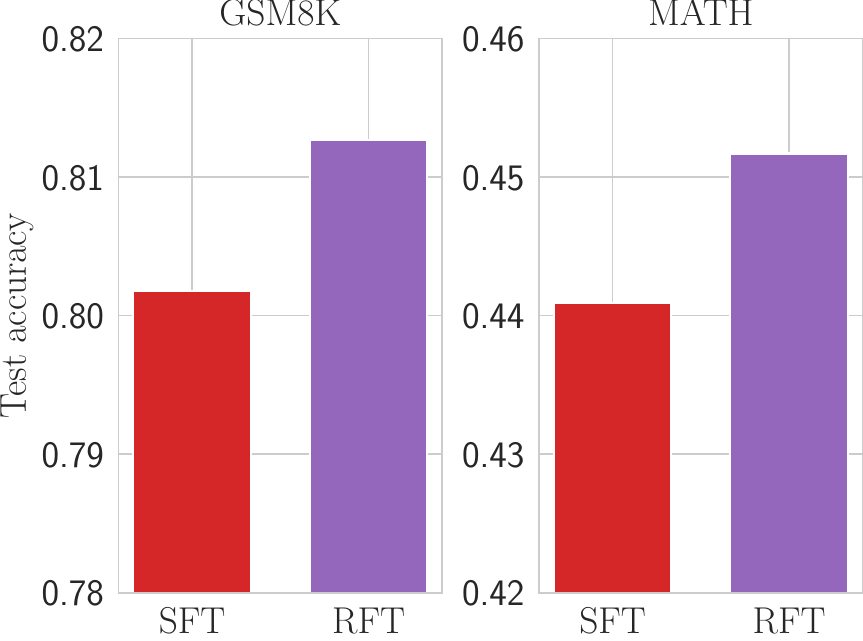}\caption{\label{fig:sft_rft_single}\footnotesize{RFT data with a single (self-generated) correct solution per problem outperforms SFT data (from highly-capable models) of the same size. 
  }}
  \vspace{-0.35cm}
\end{wrapfigure}
Note that one may also attribute better performance of RFT to improved coverage from multiple responses in $\rftdata$ for each question in $\syndata$. We find that even when RFT data is restricted to one solution per question, base LLMs finetuned on it outperform SFT  consistently by $>1\%$. 
In Figure~\ref{fig:sft_rft_single}, we plot the performance of DeepSeek-Math-7B finetuned on SFT data $\syndata$ and RFT data $\rftdata$ where $\rftdata$ has the same questions as $\syndata$, and only one positive solution per question, sampled from $\pi_\sft$ finetuned on $\syndata$. Thus, both SFT and RFT datasets are of the same size. 
This means that a significant portion of the $2\times$ sample efficiency gains we observe for RFT in Figure~\ref{fig:pos-scaling}(a,b) can  be attributed to RFT data from $\pi_\sft$ being easier-to-fit, and not purely because RFT data improves coverage by finetuning on multiple solution traces per question.

\textbf{SFT/RFT policy suffers from spurious correlations in positive synthetic data.} 
While RFT data maybe ``easier-to-fit'', in Figure~\ref{fig:pos-scaling}(c) we also note that continuing to scale RFT data leads to test error saturation, or even worse test error. This is unlike scaling of problems and solutions in SFT data (in Figure~\ref{fig:pos-scaling}(a,b)). 
\begin{wrapfigure}{r}{0.4\textwidth}
\captionsetup{font=footnotesize,skip=0pt}
\footnotesize
  \centering
  \vspace{-0.3cm}
  \includegraphics[width=0.99\linewidth]{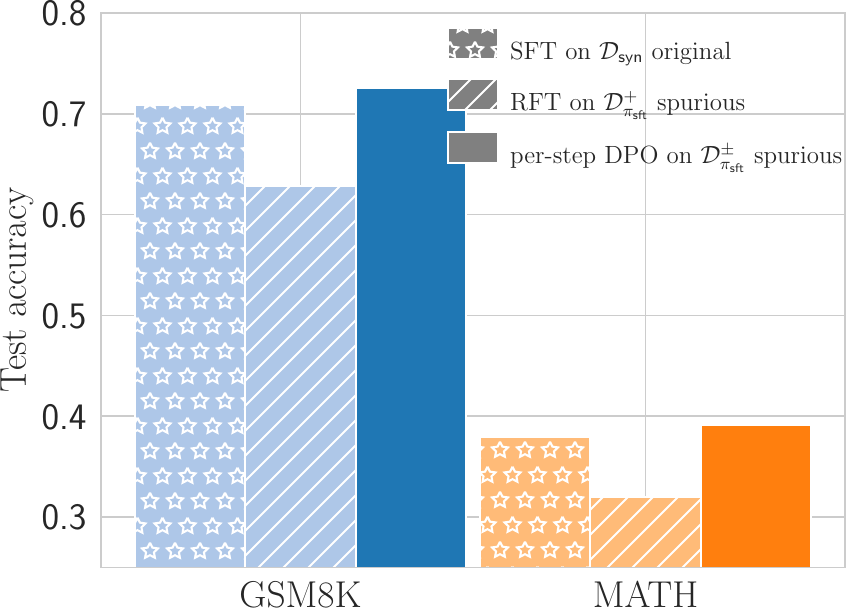}
  \caption{Running RFT on positive model-generated  data $\rftdata$ with spurious steps can be worse than running SFT on $\syndata$. Later we will discuss how this failure mode is fixed by per-step credit assignment with per-step DPO.}
  \vspace{-0.38cm}
  \label{fig:spurious}
\end{wrapfigure}
This failure can be attributed to the presence of incorrect/irrelevant steps that are not detected by our verifier, since it only verifies the final answer (see Appendix~\ref{sec:implementation},~\ref{app:spurious-examples} for examples). For a problem $\bx$, when the 
LLM is trained with supervised next-token prediction on some positive sub-optimal $\by$ in the RFT data, with incorrect or irrelevant step $\by_{k}$, it is likely to develop incorrect causal associations between such ``spurious'' steps appearing in a prefix of the solution $\by_{1:k}$, and the following valid step $\by_{k+1}$, when trying to maximize $\pi(\by_{k+1} | \by_{1:k}, \bx)$. This is well known as the problem of ``causal confusion''~\citep{de2019causal} in the literature on imitation learning. To verify this hypothesis, we amplify the presence of these spurious steps. Specifically, for each question in $\syndata$ we sample ``spurious steps'' from $\pi_\sft$ trained on it, \ie steps which lead to the incorrect answer with high probability under $\pi_\sft$ (we sample multiple completions conditioned on the same spurious step to check how likely it leads to the correct final answer). Then, we interleave the solution traces in the RFT data with these spurious steps. Note, that all traces in the RFT data are still positive since, they all lead to the correct answer eventually. We find that an LLM trained on this spurious RFT data performs worse than the $\pi_\sft$ policy itself, often producing more of such incorrect steps in our qualitative probes (Appendix~\ref{app:spurious-examples}). 

\begin{AIbox}{Takeaways for scaling positive synthetic data}
\begin{itemize}[leftmargin=0em,noitemsep]
    \item While positive data from GPT-4/Gemini-1.5 improves coverage over new problems and solutions, positive data from SFT policy trained on it is $2\times$ more sample efficient.
    \item When positive data from $\pi_\sft$ contains spurious steps, scaling synthetic data leads to worse test errors.
\end{itemize}
\end{AIbox}

\vspace{-0.2cm}
\section{Negative Synthetic Data Enables Per-Step Credit Assignment}
\label{sec:negative-synth-expts}
\vspace{-0.1cm}
We saw in Section~\ref{sec:positive-synth-expts} that fine-tuning on model-generated data can reinforce the model's tendency to produce irrelevant or incorrect steps that are able to still steer the model towards the correct response on some training problems, but derail it otherwise. Since supervised fine-tuning on positive data with a low-cost verifier alone is unable to identify such problematic steps either, we would not expect this issue to go away even as the amount of synthetic positive data is scaled up substantially. 

Perhaps surprisingly, we show that training on \textbf{\emph{negative data}} for the same set of synthetic problems can address the failure modes of positive data as long as this negative data is carefully constructed. In particular, the construction of negatives should enable us to perform \emph{per-step credit assignment}. We will develop a conceptual model to do so in this section and build a recipe for obtaining per-step credits. 
We will formally show how DPO on a specific preference pair distribution implicitly conforms to our model of per-step credits.
We also show that per-step DPO  from Section~\ref{sec:preliminaries} is an approximate variant of this more general conceptual idea. Consequently, using per-step DPO we empirically verify the efficacy of negative data in enabling per-step credit assignment and its impact on the scaling laws.
Finally, we discuss how our conceptual model of using negatives is equivalent to advantage-weighted RL and explain the generalization benefits of using negative data by exploiting this connection, both theoretically and intuitively via a study on a didactic problem.

\vspace{-0.15cm}
\subsection{Conceptual Model: Constructing Negatives to Enable Per-Step Credit Assignment}
\label{sec:conceptua_model}
\vspace{-0.15cm}

While na\"ively contrasing an entire positive response $\hatposby$ against an entire negative response $\hatnegby$ will increase the likelihood of \emph{each} step that appears in $\hatposby$ (even when incorrect or irrelevant) and reduce likelihood on each step appearing in $\hatnegby$ (even when accurate and relevant), it does not account for the importance of each step. Formally, given a negative solution trace $\hatnegby$, we would want to identify \emph{critical} steps where the model introduces a flaw $\hatnegby_i$, and emphasize alternate correct completions starting from the previous state $(\bx, \hatnegby_{1:i-1})$, that the model could have still produced. Likewise, given a positive solution trace, $\hatposby$, we would like to identify if a given step $\hatposby_{i}$ does not make progress towards the solution by identifying if there exist more likely to succeed alternatives from its predecessor steps, $\hatposby_{1:i-1}$, which now presents a key decision-making point. \textbf{What are these critical steps and how can we identify them procedurally?}

\textbf{Value functions.} We formalize the notion of per-step credit using value functions from reinforcement learning (RL). Recall that both $\hatposby$ and $\hatnegby$ are sampled from $\pi_\sft$. For problem $\bx$, with correct solution $\by$, a response $\hatby$ with a sequence of steps $\hatby_{1:i-1}$ (state), and a candidate step $\hatby_{i}$ (action), we define the  Q-function under policy $\tilde{\pi}$ in Equation~\ref{eq:optimal_value_function}. 
\begin{align}
    \label{eq:optimal_value_function}
    Q_{\tilde{\pi}}(\underbrace{\bx, \hatby_{1:i-1}}_{\text{state}}; \underbrace{\hatby_{i}}_{\text{action}} ) =\underbrace{
    \mathbb{E}_{{\by^{\mathrm{new}}_{i+1:L} \sim \tilde{\pi}(\cdot|\bx, \hatby_{1:i})}} \Big[ r\left( [\hatby_{1:i}, \by^{\mathrm{new}}_{i+1:L}], \by \right) \Big]}_{\text{expected future reward under new actions (i.e., steps) sampled by policy }\tilde{\pi}}
\end{align}
To interpret this expression, for any partial solution upto $i-1$ steps (current state: $(\bx, \hatby_{1:i-1})$), and the next step $i$ (next action: $\hatby_{i}$), the above Q-function evaluates the probability of success given the remaining budget of $L - i$ more steps, in expectation over all possible futures sampled from some policy $\tilde{\pi}$. Our conceptual model treats the policy $\tilde{\pi}$ as an algorithmic design choice that can differ for algorithms using negative data. As we see later, choosing $\tilde{\pi}$ as the best-of-K distribution around $\pi_\sft$ (denoted as $\bok(\pi_\sft)$) enables a particularly interesting tradeoff between $Q$-value estimation and policy improvement. Another common choice is $\pi_\sft$ itself. 
Now, for any given step $\hatby_{i}$, we can define its \textbf{\emph{advantage}} as the relative change in $Q_{\tilde{\pi}}$ when adding step $\hatby_{i}$ in comparison with other possible candidates for step $i$ as follows:
\begin{align}
    \label{eq:advantage}
    A_{\tilde{\pi}}(\bx, \hatby_{1:i-1}; \hatby_{i}) = Q_{\tilde{\pi}}(\bx, \hatby_{1:i-1}; \hatby_{i}) - Q_{\tilde{\pi}}(\bx, \hatby_{1:i-2}; \hatby_{i-1}). 
\end{align}
Equation~\ref{eq:advantage} is identical to the definition of advantage of an action (\ie $\hatby_{i}$) at a state ($\bx, \hatby_{1:i-1}$) from reinforcement learning~\citep{suttonrlbook}. It is the gap between the Q-value of a state-action pair and the value function of the state (which is equal to the Q-value of the \emph{previous} step due to deterministic dynamics).

\begin{wrapfigure}{r}{0.71\linewidth}
    \vspace{-0.37cm}
\captionsetup{font=footnotesize,skip=0.5pt}
    \includegraphics[width=0.99\linewidth]{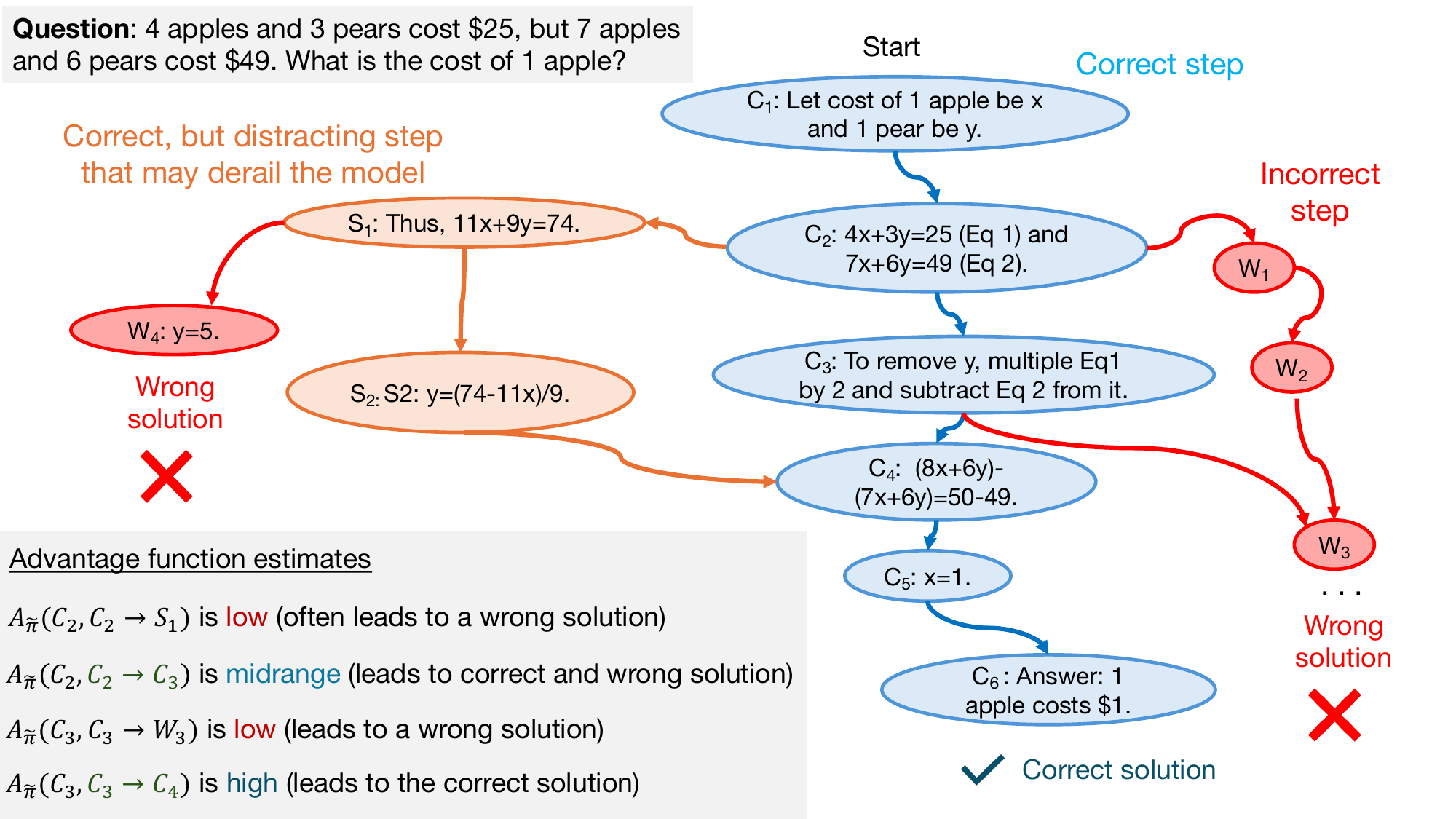}
    \caption{\footnotesize{Illustration of advantage estimation from negative data on a didactic example in synthetic model generations. Critical steps are those with high advantage values.
    }\vspace{-0.25cm}}
    \label{fig:critical_states}
\end{wrapfigure}
\textbf{Critical steps, per-step DPO, and advantage-weighted RL.} We can use advantages (Equation~\ref{eq:advantage}) to characterize critical steps. 
\textcolor{lightblue}{{Steps that attain a higher advantage value than others are \textbf{critical} since they need to be generated more precisely to solve the problem. In contrast, steps that with very low advantage values are likely worse and must be unlearned.}}
Our definition of the advantage function implies that one can calculate advantages for each step in a response via additional Monte Carlo rollouts starting from prefixes defined by partial solutions. 
One could then use these advantage estimates (Equation~\ref{eq:advantage}) for training the model, for example, by running advantage-weighted reinforcement learning~\citep{peng2019advantage}. An alternate option would be to skip the computation of advantage estimates but instead rely on implicit approaches that optimize the advantage-weighted objective without computing their values. Theorem~\ref{thm:equivalence} shows that DPO performed over a precise pair distribution contrasting positive and negative traces obtained via additional rollouts from $\tilde{\pi}$, on prefixes of a response sampled from $\pi_\sft$ 
is equivalent to advantage-weighted RL.
A proof of Theorem~\ref{thm:equivalence} is in Appendix~\ref{sec:proof-ofequiv}. Note that unlike the standard reduction of DPO to the RL objective under \emph{some} reward function~\citep{rafailov2023direct,rafailov2024r}, Theorem~\ref{thm:equivalence} is stronger in that it identifies the value function induced by per-step DPO. 
\begin{theorem}[Equivalence of advantage-weighted RL and DPO with per-step pairs]
\label{thm:equivalence}
    The optimal policy from Equation~\ref{eq:dpo_objective} with $\dpodata$ given by $(\bx, [\by_{1:i}, \posby_{i+1}], [\by_{1:i}, \negby_{i+1}])$ where the positive and negative traces share prefix $\by_{1:i} \sim \pi_\sft$, and $\negby_{i+1}\sim \pi_\sft(\cdot | \bx, \by_{1:i})$, $\posby_{i+1} \sim \sigma(A_{\tilde{\pi}}(\bx, \by_{1:i}; \cdot)- A_{\tilde{\pi}}(\bx, \by_{1:i}; \negby_{i+1}))$,  
    is identical to the optima of the advantage-weighted RL objective:
    \begin{align}
        \max_{\pi}~~& \mathbb{E}_{\bx \sim p_\mathrm{syn}(\bx), \by  \sim \pi_\sft(\cdot|\bx)} \left[ \sum_{i=1}^L \log \pi(\by_{i} \big| \bx, \by_{0:i-1}) \cdot \exp \left( A_{\tilde{\pi}}(\bx, \by_{0:i-1}; \by_{i}) / \beta \right) \right]. \vspace{-0.5cm}
        \label{eq:advanage_weighted_objective}
    \end{align}
\end{theorem}

\textbf{Practical instantation of DPO with per-step pairs.} In most of our experiments, we instantiate a practical version of the above framework, following the scheme in \citet{hwang2024self}. This is a special case (Part 1) of the complete algorithm shown in Algorithm~\ref{alg:per-step-dpo}. Unless otherwise mentioned, we use ``per-step DPO'' to refer to this version (Part 1 only) in practice. We will also experiment with the complete version (parts 1 and 2) later in Section~\ref{sec:Exp_with_part2}.
Instead of computing advantage estimates for each step, and then sampling preference pairs, as described in Theorem~\ref{thm:equivalence}, we approximate this by only Q-value estimates on $8$ negative responses for each question in the synthetic dataset,
with $\tilde{\pi}$ chosen to be the best-of-K policy, $\bok(\pi_\sft)$ where $K=5$. There are two benefits associated with this choice of $\tilde{\pi}$, especially a higher value of $K$: \textbf{(i)} estimating the advantage in 
 Equation~\ref{eq:advantage} with Monte Carlo rollouts exhibits lower variance when $K$ is large, since a larger budget $K$ would lead most steps to have higher Q-values and the variance of Bernoulli reduces as Q-value $\rightarrow$ $1$;
 and \textbf{(ii)} $Q_{\bok({\pi_\sft})}$ is a non-decreasing function in $K$ for any state-action, which implies that the solution of advantage-weighted RL objective, in principle, can now improve over a better policy $\bok({\pi_\sft})$, compared to $\pi_\sft$. 
 We will next discuss scaling results for negative data, and then in Section~\ref{sec:generalization} show how per-step credit assignment improves generalization and suppresses irrelevant and incorrect steps appearing in a response, making it possible to make more use of the same synthetic data. 
 {
\begin{algorithm}[H]
\caption{Per-step DPO (\textcolor{lightblue}{Part 1:} Practical version for most experiments; \textcolor{lightblue}{Parts 1 + 2:} Complete version)}
\label{alg:per-step-dpo}
\begin{algorithmic}[1]
\Require{Synthetic dataset: $\syndata$, SFT policy finetuned on $\syndata$: $\pi_\sft$, sampling policy $\tilde{\pi}$.}
\State Initialize per-step DPO dataset $\dpodata \gets \{\}$.
\For {$(\bx, \by) \in \syndata \cup \rftdata$}
    \State \textcolor{lightblue}{\# \textbf{Part 1: }Identify critical steps in incorrect responses}
    \State Sample multiple incorrect answers $\hatnegby \sim \pi_\sft(\cdot \mid \bx)$, and collect them in set $\mathcal{C}(\bx)$.
    \For {$\hatnegby := \brck{\hatnegby_1, \ldots, \hatnegby_L} \in \mathcal{C}(\bx)$}
        \State Compute the Monte Carlo estimate for $Q_{\tilde{\pi}}(\bx, \hatnegby_{1:i-1}; \hatnegby_i)$ for each step $\hatnegby_i$.
        \State If $\hatnegby_c$ is the first step with least $Q_{\tilde{\pi}}(\bx, \hatnegby_{1:i-1}; \hatnegby_i)$, then $\dpodata \gets \dpodata \cup \{(\bx, \by, \hatnegby_{1:c})\}$.
    \EndFor
    \State \textcolor{lightblue}{\# \textbf{Part 2: }Identify spurious steps in correct responses}
    \State Sample multiple correct answers $\hatposby \sim \pi_\sft(\cdot \mid \bx)$, and collect them in set $\mathcal{C}'(\bx)$.
    \For {$\hatposby := \brck{\hatposby_1, \ldots, \hatposby_L} \in \mathcal{C}'(\bx)$}
        \State Compute the Monte Carlo estimate for $Q_{\tilde{\pi}}(\bx, \hatposby_{1:i-1};\hatposby_i)$ for each step $\hatposby_i$.
        \State If $\hatposby_c$ is the first step with least $Q_{\tilde{\pi}}(\bx, \hatposby_{1:i-1}; \hatposby_i)$, then $\dpodata \gets \dpodata \cup \{(\bx, \by, \hatposby_{1:c})\}$.
    \EndFor
\EndFor
\State Optimize DPO loss in Equation~\eqref{eq:dpo_objective} on $\dpodata$ with $\pi_\sft$ as the reference policy.
\end{algorithmic}
\end{algorithm}
\vspace{-0.5cm}
}

\begin{figure}[t]
    \captionsetup{font=small,skip=2pt}
    \centering
    \begin{subfigure}[b]{0.33\textwidth}
        \captionsetup{font=small,skip=0pt}
        \centering
        \includegraphics[width=0.99\textwidth]{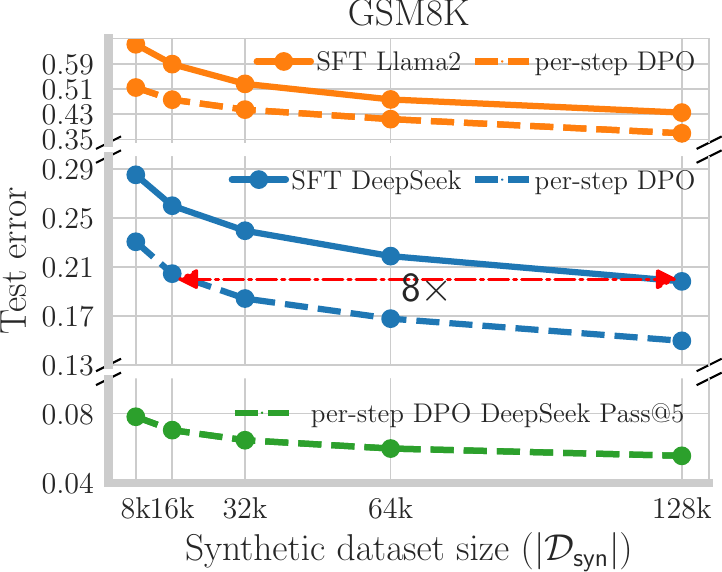}
        \caption*{(a)}
    \end{subfigure}\hfill
    \begin{subfigure}[b]{0.33\textwidth}
        \captionsetup{font=small,skip=0pt}
        \centering
        \includegraphics[width=0.99\textwidth]{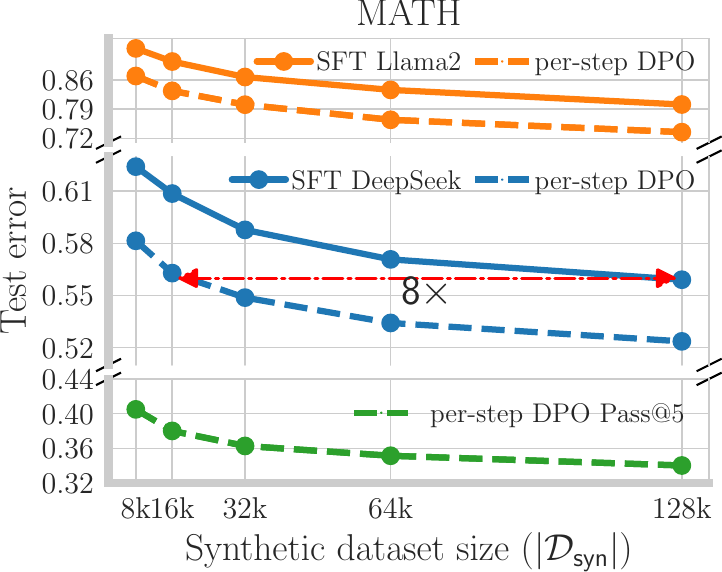}
        \caption*{(b)}
    \end{subfigure}\hfill
    \begin{subfigure}[b]{0.32\textwidth}
        \captionsetup{font=small,skip=0pt}
        \centering
        \includegraphics[width=0.99\textwidth]{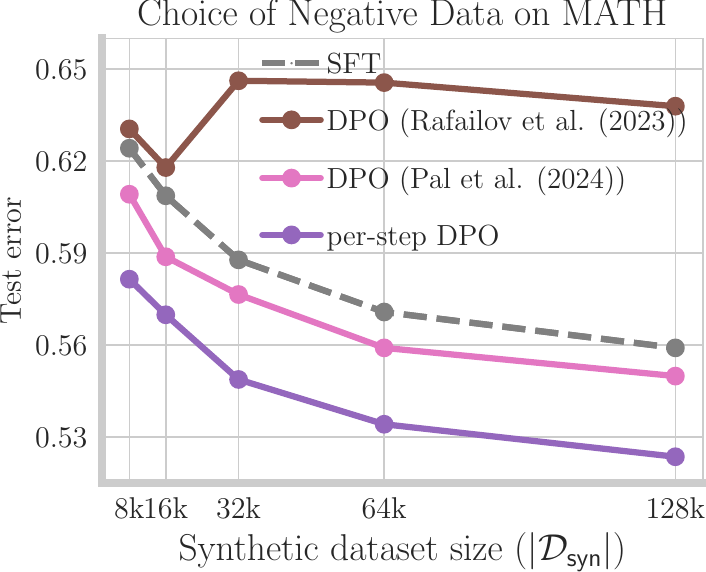}
        \caption*{(c)}
    \end{subfigure}
    \caption{\footnotesize{\textbf{\emph{Negative data scaling laws:}}} We evaluate algorithms that consume negative data as we scale $\syndata$, and compare them with only positive training (SFT) on $\syndata$. On GSM8K (a) and MATH (b), we observe an $8 \times$ gain from per-step DPO (Section~\ref{sec:learning}) which aligns with our model of negative data that enables per-step credit assignment. In (c) we compare different negative data construction algorithms, and particularly note that na\"ively pairing positives and negatives~\cite{rafailov2023direct} leads to worse performance as we scale the number of synthetic problems in $\syndata$.}
    \label{fig:neg-scaling}
    \vspace{-0.4cm}
\end{figure}

\vspace{-0.2cm}
\subsection{Scaling Results for Negative Data} 
\vspace{-0.2cm}
\label{subsec:scaling_results_neg_data}
We present the results for scaling up negative data for training the DeepSeek-Math-7B and LLama2-7B models on GSM8K and MATH in Figure~\ref{fig:neg-scaling}(a,b). 
Observe in Figure~\ref{fig:neg-scaling}(a,b), that for both DeepSeek-Math-7B and LLama2-7B models, per-step DPO improves performance beyond the SFT policy and the performance continues to scale favorably as data size increases. In fact, also note that for any given size of $\syndata$, per-step DPO also substantially improves over RFT (Figure~\ref{fig:pos-scaling}) on both datasets, and overall, \textcolor{lightblue}{\textbf{while RFT improved effective data size of $\syndata$ by $2\times$, additionally training on negative data extends the performance improvement to $8\times$ the size of $\syndata$.}} Additionally, since per-step DPO estimates advantage of each step under the Best-of-K policy (with $K=5$ in this setting), one might expect a saturation in the pass@5 performance of the per-step DPO solution. On the contrary, we find that pass@5 performance also improves consistently.

\textbf{Choice of negative data has significant impact.} 
In Figure~\ref{fig:neg-scaling}(c) we plot negative data scaling laws where the choice of negative data (and thereby pairs for DPO in Equation~\ref{eq:dpo_objective}) differs.  
Observe that standard pairing of positive and negative responses in $\dpodata$ for DPO~\citep{rafailov2023direct} does not improve upon the SFT policy. As such, we needed to tune $\beta$ in Equation~\ref{eq:dpo_objective} for DPO but could not fully avoid performance degradation. 
Our conceptual model explains this result: since contrasting arbitrary positives and negatives would result in an incorrect induced advantage function, training with DPO will exacerbate spurious correlations that maximize this induced advantage function~\cite{saeidi2024insights,pang2024iterative,xu2024contrastive}. In fact, \citet{pal2024smaug} also find similar concerns with random pairing and instead pair positives and negatives that with highest edit distance, which leads to some improvement over standard DPO (Figure~\ref{fig:pos-scaling}(c)) but still performs poorer than per-step DPO that accounts for credit.

\begin{AIbox}{Takeaways for scaling negative synthetic data}
\begin{itemize}[leftmargin=0em,noitemsep]
    \item Negative data can identify high-advantage (critical) steps in model-generated responses.
    \item  We can construct negative data distribution that equates DPO to advantage-weighted RL. Negative data used in this way improves the sample efficiency of synthetic data by $8 \times$. 
\end{itemize}
\end{AIbox}

\vspace{-0.15cm}
\subsection{Why Does Credit Assignment from Negative Data Improve Model Generalization?} 
\label{sec:generalization}
\vspace{-0.15cm}
Recall that when we sample preference pairs from the  distribution in Theorem~\ref{thm:equivalence}, per-step DPO is equivalent to advantage-weighted RL. 
In other words, our conceptual model illustrates that per-step DPO can perform credit assignment, 
and identify critical steps over irrelevant or incorrect ones by computing advantage estimates.
We see that this improves test performance and scaling trends (in Section~\ref{subsec:scaling_results_neg_data}) over SFT and standard DPO~\cite{rafailov2023direct} with the same data. To explain these results, we need to understand the generalization gains afforded by the advantage-weighted RL objective.

\vspace{-0.2cm}
\subsubsection{Advantage-Weighted RL De-Emphasizes Spurious Steps and Emphasizes Critical Steps}
\vspace{-0.2cm}

We first explain how \emph{online} model-specific interventions and advantage estimation would emphasize critical steps and de-emphasize spurious ones. Consider $\tilde{\pi} = \pi_\sft$, \textit{i.e.}, the advantage estimates are computed under
 \begin{wrapfigure}{r}{0.4\textwidth}
 \vspace{-0.25cm}
  \captionsetup{font=footnotesize,skip=3pt}
  \centering
  \includegraphics[width=0.99\linewidth]{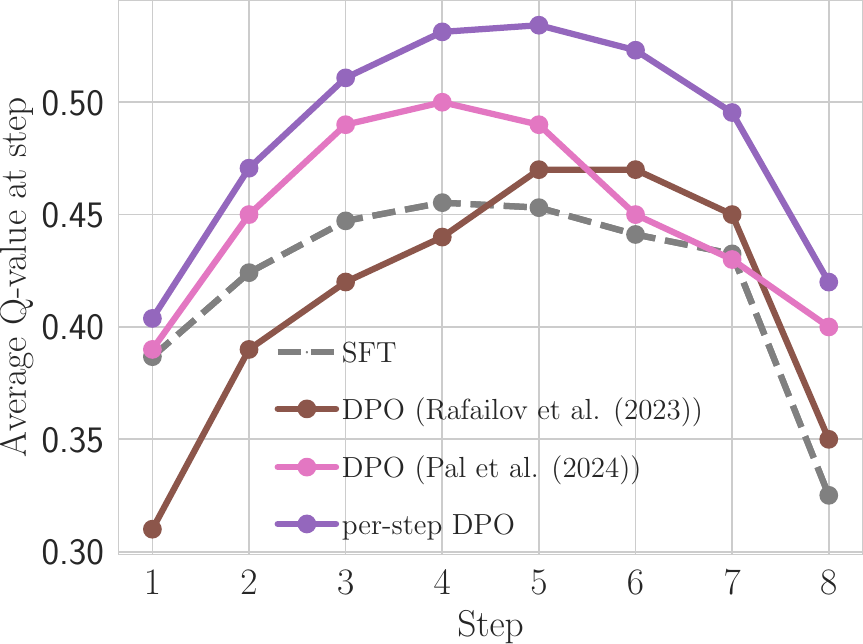}
  \caption{\label{fig:qvals} \footnotesize{Per-step DPO improves Q-values at each step, standard DPO only improves at irrelevant steps.\vspace{-1.0em}}}
\end{wrapfigure}
the base policy itself. If some spurious (\textit{i.e.}, incorrect or distracting) step $\by_s$ is imperfectly cloned under $\pi_\sft$, then while the average teacher-forcing loss is misleadingly low, the loss on just the spurious step would remain high. This would mean that when using autoregressive decoding during inference, the rollouts may  not replicate steps in the training data perfectly. Due to this imperfect replication, spurious steps will present low advantage estimates: concretely, if a step $\by_s$ is imperfectly cloned and enough samples are drawn from $\pi_\sft$ conditioned on $\by_{1:s}$, many of them would lead to incorrect final answers. On the other hand, if all steps in the SFT data are perfectly cloned, then during inference, the rollouts would reproduce the solutions that end with correct answers. Thus, all steps in this solution would have a higher advantage estimates. Thus in this case, training with RL would de-emphasize spurious steps and emphasize critical steps. 

To verify this explanation in practice, we recall that running per-step training with negative data generated by the RFT model that has overfit on distracting steps improved accuracy by $>$$6\%$ in Figure~\ref{fig:spurious}. We visualize advantages against real examples appearing in this fine-tuning dataset in Appendix~\ref{app:spurious-examples}, and observe that advantage estimates indeed de-emphasize spurious steps. Aggregating over examples, in Figure~\ref{fig:qvals}, we plot the average Q-value of a step for different negative data schemes, and note that only per-step DPO improves the Q-values attained at every step compared to SFT, as expected based on the connection to advantage-weighted RL (Theorem~\ref{thm:equivalence}), indicating that per-step DPO trains the model to emphasize the more critical steps.

\begin{figure}[t]
    \centering
    {\footnotesize

\begin{tcolorbox}[colback=blue!5!white, colframe=blue!75!black, title=Advantage Estimates Computed Over a Positive Model Generated Sample]
\textbf{Question:} 

Betty wants a new wallet which costs \$100. Betty has only half of the money she needs. Her parents give her \$15 for that purpose, and her grandparents twice as much as her parents. How much more money does Betty need?

\textbf{Positive response with identified spurious step: }

{\setlength{\fboxsep}{0pt}\colorbox{blue!5!white}{\parbox{\textwidth}{
\colorbox{red!19.9}{\strut Betty's} \colorbox{red!19.9}{\strut parents} \colorbox{red!19.9}{\strut gave} \colorbox{red!19.9}{\strut her} \colorbox{red!19.9}{\strut 15,} \colorbox{red!19.9}{\strut so} \colorbox{red!19.9}{\strut her} \colorbox{red!19.9}{\strut grandparents} \colorbox{red!19.9}{\strut gave} \colorbox{red!19.9}{\strut her} \colorbox{red!19.9}{\strut 2} \colorbox{red!19.9}{\strut *} \colorbox{red!19.9}{\strut 15} \colorbox{red!19.9}{\strut =} \colorbox{red!19.9}{\strut <<15*2=30>>30.} \colorbox{red!16.94}{\strut In} \colorbox{red!16.94}{\strut total,} \colorbox{red!16.94}{\strut Betty} \colorbox{red!16.94}{\strut received} \colorbox{red!16.94}{\strut 30} \colorbox{red!16.94}{\strut +} \colorbox{red!16.94}{\strut 15} \colorbox{red!16.94}{\strut =} \colorbox{red!16.94}{\strut <<30+15=45>>45.} \colorbox{red!62.96}{\strut The} \colorbox{red!62.96}{\strut total} \colorbox{red!62.96}{\strut amount} \colorbox{red!62.96}{\strut of} \colorbox{red!62.96}{\strut money} \colorbox{red!62.96}{\strut Betty} \colorbox{red!62.96}{\strut needs} \colorbox{red!62.96}{\strut is} \colorbox{red!62.96}{\strut 100} \colorbox{red!62.96}{\strut *} \colorbox{red!62.96}{\strut 2} \colorbox{red!62.96}{\strut =} \colorbox{red!62.96}{\strut <<100*2=200>>200.} \colorbox{red!23.31}{\strut Betty} \colorbox{red!23.31}{\strut needs} \colorbox{red!23.31}{\strut 100} \colorbox{red!23.31}{\strut /} \colorbox{red!23.31}{\strut 2} \colorbox{red!23.31}{\strut =} \colorbox{red!23.31}{\strut <<100/2=50>>50} \colorbox{red!23.31}{\strut in} \colorbox{red!23.31}{\strut total.} \colorbox{red!37.06}{\strut Betty} \colorbox{red!37.06}{\strut still} \colorbox{red!37.06}{\strut needs} \colorbox{red!37.06}{\strut to} \colorbox{red!37.06}{\strut save} \colorbox{red!37.06}{\strut 50} \colorbox{red!37.06}{\strut -} \colorbox{red!37.06}{\strut 45} \colorbox{red!37.06}{\strut =} \colorbox{red!37.06}{\strut <<50-45=5>>5.} \colorbox{red!24.35}{\strut The} \colorbox{red!24.35}{\strut answer} \colorbox{red!24.35}{\strut is} \colorbox{red!24.35}{\strut 5} 
}}}
\end{tcolorbox}
    }
    \vspace{-1em}
\caption{\footnotesize\textit{\textbf{Spurious steps have low advantage values:}} For each step in a correct (positive) response sampled from the SFT policy we compute the advantage values (lower values are darker) and find that the advantage estimates indeed identify the spurious step. \vspace{-1em}}
    \label{fig:enter-label}
\end{figure}

\vspace{-0.2cm}
\subsubsection{Why Does Generalization Improve?: Connecting Advantage-Weighted RL to DRO}
\label{sec:dro}
\vspace{-0.15cm}

In the previous section, we discussed how advantage-weighted RL preferentially weighs the next-token prediction loss at each step. Now, we attempt to conceptually understand why this could improve generalization. For this, we present an intuitive explanation by drawing a connection between advantage-weighted RL and a distributionally robust optimization (DRO) algorithm, named Group DRO, commonly used to improve worst-group robustness in supervised learning~\cite{rahimian2019distributionally}.

\textbf{Intuitive explanation.} During inference, the SFT policy can fail even on training problems, especially in scenarios where the SFT policy has failed to perfectly clone the next step at each intermediate step in the SFT data. As previously discussed, these steps also present with low advantage values. One way to reduce the chance of compounding inference time errors~\cite{ross2010efficient} is to preferentially minimize the negative log-likelihood loss \emph{more} for the critical step, i.e., those steps from where the model is more likely to arrive at a wrong answer. If we iteratively update the policy with gradient steps computed over a re-weighted next-step prediction objective where each step is weighted by its advantage estimate, then the resulting algorithm intuitively exhibits this characteristic similarly to distributionally robust optimizers (DRO)~\cite{levy2020large}.
Similar to how DRO solutions guarantee that all subpopulations -- both majority and minority subpopulations -- in the training data achieve low loss values, the solution for the advantage-weighted RL objective guarantees that the negative log-likelihood of the critical steps with high advantage estimates under $\tilde{\pi}$ (which of per-step DPO is $\bok(\pi_\sft)$) is also low, to a similar extent as the other more prevalent non-critical steps.

In other words, \textbf{our insight} is that weighting steps using advantages in Equation~\ref{eq:advanage_weighted_objective} upweights the likelihood of the underrepresented critical states while down-weighting it for the spurious ones. 
The guarantees on the training data ($\syndata$) also translate to the population level objective when the weights for on-policy samples (advantage estimates) are accurate~\cite{setlur2023bitrate} and the policy is sufficiently regularized~\cite{sagawa2019distributionally}.
Since correct behavior at critical steps determine the correctness of the overall solution, an elevated degree of correctness at executing critical steps at the population level implies a higher test accuracy on the reasoning task.

\vspace{-0.2cm}
\subsubsection{But, Attaining Low Generalization Error Requires Low Advantage Estimation Error}
\label{sec:Exp_with_part2}
\vspace{-0.2cm}
The practical efficacy of algorithms that use negative data for credit assignment requires the advantage estimates to be fairly accurate when estimated using just a few rollouts from $\tilde{\pi}$. For understanding this, let us consider a simple example when we set $\tilde{\pi} = \pi_\sft$.
If the initial advantage of a spurious step is incorrectly over-estimated, advantage-weighted RL via negative data would up-weight the likelihood further. To see why, consider a scenario when most Monte Carlo rollouts from $\tilde{\pi} = \pi_\sft$ contain the spurious step. 
Since the model is still able to generate the correct final answer even after it produces a spurious step on training problems, subsequent rollouts from the model would estimate higher $A_{\pi_\sft}$ for this step, and further advantage-weighted RL training would emphasize this step. This would result in a downward spiral: the model is trained to use the spurious step, which then results in even higher advantage values for this step, which in turns reinforces the same step over and over again. In this scenario, we would not expect advantage-weighted RL to unlearn the spurious step at all or improve generalization, in fact it would only amplify the dependence on this step.
On the other hand, when $\tilde{\pi} = \bok(\pi_\sft)$ for a higher value of $K$, the Monte Carlo advantage estimator has a lower variance. This justifies the choice of $K=5$, which is large enough given our computational budget.

 \begin{wrapfigure}{r}{0.4\textwidth}
 \vspace{-0.25cm}
  \captionsetup{font=footnotesize,skip=0pt}
  \centering
  \includegraphics[width=\linewidth]{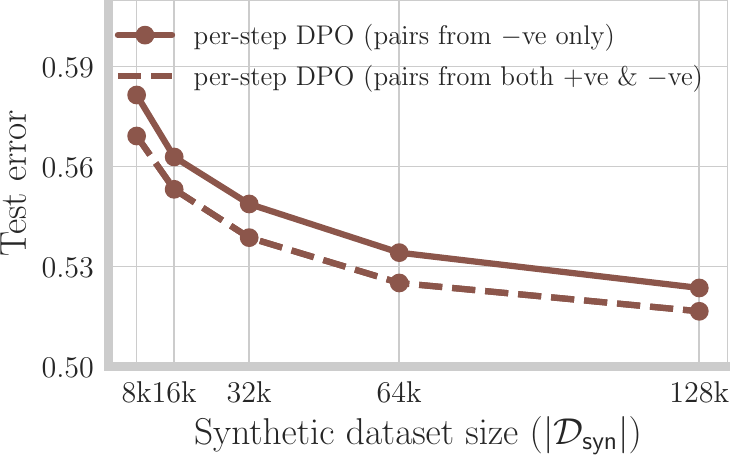}
  \vspace{-0.3cm}
  \caption{\label{fig:chosen_and_reject_dpo} \footnotesize{On MATH, improving advantage estimates by computing advantages over both positive and negative traces sampled from $\pi_\sft$ improves estimation error and final performance for per-step DPO.\vspace{-1.5em}}}
\end{wrapfigure}
\textbf{Improving per-step DPO~\cite{hwang2024self}  with a closer approximation of advantage-weighted RL (Parts 1 + 2 in Algorithm~\ref{alg:per-step-dpo}).} Now, we discuss an experiment that improves the performance of per-step DPO~\cite{hwang2024self} by running the full version of the per-step DPO algorithm (Algorithm~\ref{alg:per-step-dpo}). In particular, we add new preference pairs to the dataset of per-step DPO algorithm starting from positive samples.
Recall from Section~\ref{sec:learning} and Algorithm~\ref{alg:per-step-dpo}, that for a problem $\bx$, with correct answer $\by$ given by SFT or RFT data,  Part 1 of per-step DPO estimates the expected accuracy (Q-value) of each step in a negative rollout $\hatnegby \sim \pi_\sft(\cdot \mid \bx)$ sampled from the SFT policy. For each step $\hatnegby_c$ the Q-value is computed conditioned on $\bx$ and previous steps $\hatnegby_{1:c-1}$.  It then adds the triplet $(\bx, \by, \hatnegby_{1:c})$ to the DPO dataset.
We improve the coverage and accuracy of advantage estimates via Part 2, i.e., repeating this process for steps appearing on a positive trace $\hatposby \sim \pi_\sft(\cdot \mid \bx)$ as well. Specifically, we add  $(\bx, \by, \hatposby_{1:c})$  to the DPO dataset, where the step $\hatposby_{c}$ is the first step in the positive trace to have a low Q-value (as dicated by a relative threshold on the Q-value), which indicates that $\hatposby_{c}$ is likely a spurious step that the SFT model generated. For individual steps that are more likely to occur in either positive or negative traces we improve coverage of alternate possible steps, and for steps that occur in both positive and negative traces, we lower the variance (and error) of the advantage estimate. 
In Figure~\ref{fig:chosen_and_reject_dpo}, we compare the performance of per-step DPO runs with the datasets constructed from only negative vs. both positive and negative traces, and find that the latter has a lower test error for all sizes of $\syndata$.

\vspace{-0.2cm}
\subsubsection{Validating Claims About Generalization: Controlled Analysis on a Didactic Problem} 
\vspace{-0.15cm}
With the above insights and intuitive arguments, we now study the influence of $\pi_\sft$ on the generalization effects of per-step DPO. For our analysis, we consider a didactic star graph problem (Appendix~\ref{sec:star_graph_appendix}) from \citet{bachmann2024pitfalls}, where given a graph in the shape of a star and a query (center/end node), the model is asked to output the full path between the start/end nodes. This task highlights the failure of SFT at planning problems (akin to math reasoning). They show that $\pi_\sft$ minimizes SFT loss by memorizing the ``hard-to-predict'' node adjacent to the center, and copying the rest from the input graph.
It is clear that the failure stems from not being able to identify the critical adjacent token. We will show how credit assignment with negative data accurately upweights the critical token and unlearns the memorized token. 
The advantage value is estimated online, \ie at each time step we sample a batch of data (graphs), rollout steps from the current policy, and estimate the advantage for each step. This allows us to exactly sample from the per-step DPO pair distribution in Theorem~\ref{thm:equivalence}.
As we will see, the initialization for per-step DPO has a significant role to play in the generalization gains we observe. 
To vary the choice of $\pi_\sft$, we choose several intermediate checkpoints obtained during supervised finetuning for synthetic negative data generation. We consider three initializations: \textbf{(1)} an under-trained SFT model with a large training and test loss, and \textbf{(2)} an SFT model obtained by early-stopping based on a held-out validation set, where the validation loss is the lowest, and \textbf{(3)} an over-trained SFT checkpoint, with a low training but high validation loss. 

\begin{figure}[t]
    \captionsetup{font=small,skip=0pt}
    \centering
    \begin{subfigure}[b]{0.34\textwidth}
    \captionsetup{font=small,skip=0pt}
        \centering
        \includegraphics[width=0.99\linewidth]{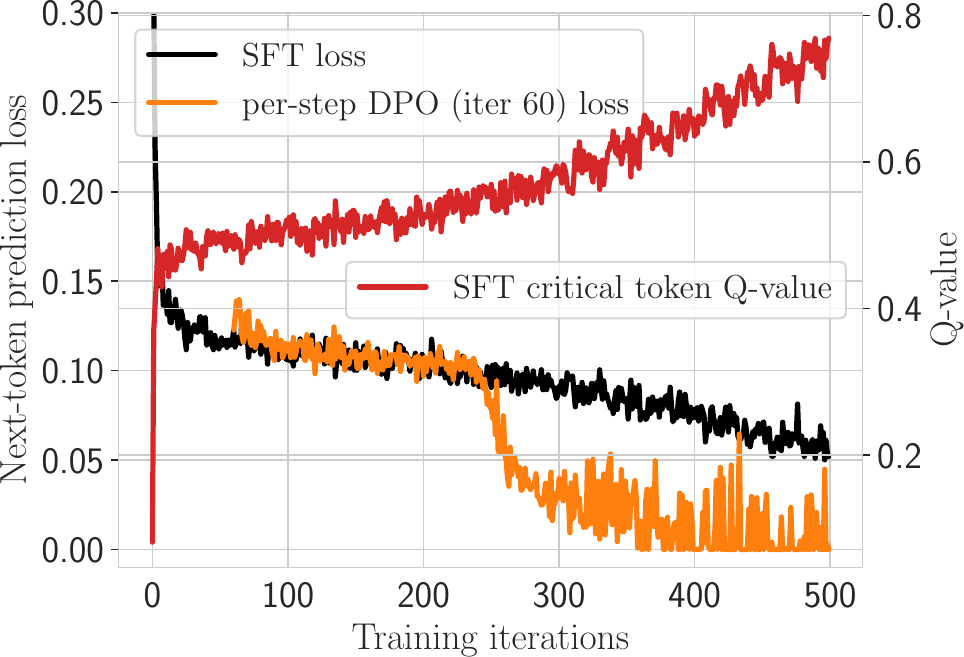}
        \caption*{(a)}
    \end{subfigure}\hfill
    \begin{subfigure}[b]{0.3\textwidth}
    \captionsetup{font=small,skip=0pt}
        \centering
        \includegraphics[width=0.99\linewidth]{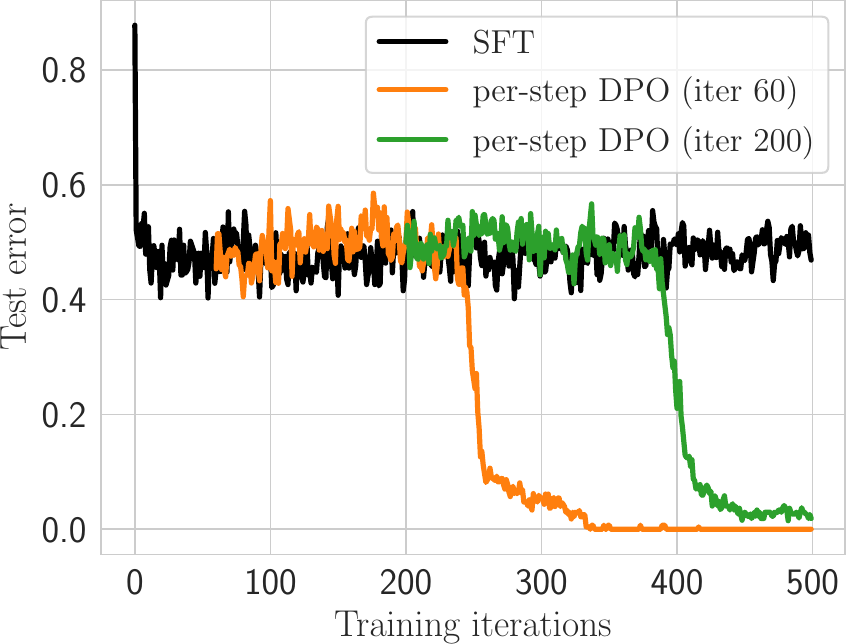}
        \caption*{(b)}
    \end{subfigure}\hfill
    \begin{subfigure}[b]{0.3\textwidth}
        \captionsetup{font=small,skip=0pt}
        \centering
        \includegraphics[width=0.99\linewidth]{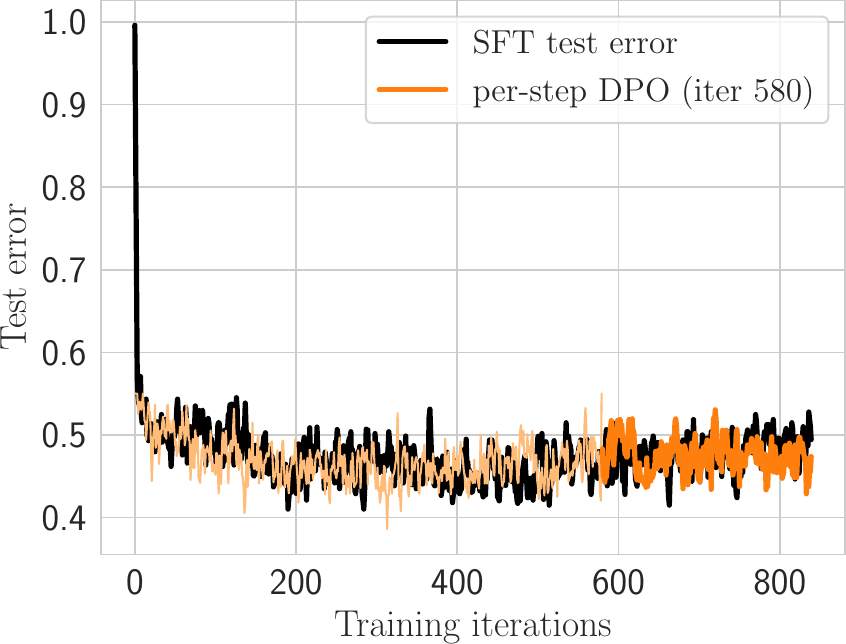}
        \caption*{(c)}
    \end{subfigure}
    \caption{\footnotesize{\textbf{\emph{Didactic analysis on star graph:}} In (a) we plot the SFT loss and Q-value of the critical token (adjacent node) for SFT and per-step DPO (starting from iter 60). Indicative of memorization SFT loss decreases at a slow rate, matching the slow rate of increase in the Q-value. In contrast per-step DPO loss sharply decreases during training. In (b) we notice a corresponding phase transition in the test error of per-step DPO starting from different under-trained SFT checkpoints, which does not happen for an over-trained SFT checkpoint in (c).}} 
    \label{fig:dpo_didactic}
    \vspace{-0.5cm}
\end{figure}
\textbf{(1) \& (2): Training on negative data from an under-trained or early-stopped $\pi_\sft$ improves both training loss and test performance.} As shown in Figure~\ref{fig:dpo_didactic}(a,b), we find that when training with negative data from iteration 60 (under-trained $\pi_\sft$) and iteration 200 (early-stopped $\pi_\sft$), utilizing per-step DPO reduces the training loss very aggressively. These benefits translate to test losses and performance as well (Figure~\ref{fig:dpo_didactic}(b), orange and green). In contrast, supervised fine-tuning exhibits a nearly-flat test loss landscape, although the train loss reduces slowly. Upon a closer inspection, we find that training on positive data via SFT only tends to memorize the critical token in the training data using non-generalizable features, and hence, the resulting model does not generalize to novel problems. This is corroborated by the slowly increasing Q-values of the critical token (Figure~\ref{fig:dpo_didactic}(b) red) as more gradient updates are performed. More training with SFT is unable to ``unlearn'' the non-generalizable feature behind this spurious step and is hence unable to reduce the loss function. On the other hand, per-step DPO with negative data is able to unlearn this spurious step and drives improvement, as evident by the drastic improvement on train and test. This experiment corroborates the connection between advantage-weighted RL and DRO from Section~\ref{sec:dro}. For the under-trained SFT model, the Q-value of the critical token (first node of the correct path) is not as high as the rest of the ``easy-to-predict'' tokens corresponding to nodes on the path that can be copied from the graph given the first node. When we optimize the advantage-weighted RL objective, this disparity in Q-values is reduced by re-weighting the next-token loss with advantage estimates, and training is able to preferentially minimize the loss at the critical token. In this scenario, this results in the model picking upon the ground-truth generalizable feature.

\textbf{(3) Training on negative data from an over-trained SFT initialization leads to model collapse.} When training with negative data on an over-trained $\pi_\sft$ (iteration 580) in Figure~\ref{fig:dpo_didactic}(c), we observe that both SFT and per-step DPO exhibit identical test errors since advantage estimates for the critical token are already high and similar to other tokens. Hence, per-step DPO would not re-weight this token any more than other tokens. Thus, training with per-step DPO would be similar to training with the standard SFT objective.
In particular, Monte Carlo samples from the over-trained checkpoint estimate a high advantage since Q-value is already high at iteration 500 (in Figure~\ref{fig:dpo_didactic}(a)), when the SFT policy has sufficiently memorized the training data using noisy features, training further is unable to unlearn this dependence. Hence, in this regime, negative data leads to no improvement over finetuning on only positive data.

\begin{AIbox}{Takeaways for generalization and spurious correlations with negative data}
Advantage-weighted RL can unlearn spurious steps and, hence, improves generalization when: (i) advantage estimation error is low; and (ii) the base model initialization is under-trained that imperfectly cloned spurious steps have low advantage, such that they can be identified with negative data. 
\end{AIbox}

\vspace{-0.2cm}
\section{Discussion and Conclusion}
\vspace{-0.2cm}

Our work studies the role of synthetic data for improving math reasoning capabilities of LLMs. We find that while the typical approach of collecting new questions and corresponding positive (correct) solutions from capable models like GPT-4/Gemini-1.5 presents underwhelming data scaling. The sample efficiency of the same data can be improved up to $2 \times$ by sampling more positive traces from the 7B sized models SFT-ed on the original data. However, training on positive self-generated synthetic data alone often amplifies the model's dependence on spurious steps, that erroneously appear to lead to a good solution but do not generalize to novel problems and hurt test performance.
That said, surprisingly, we show that negative (incorrect) traces sampled from the same SFT model can be used to address the failure modes of training on only positive data. In particular, negative data can be used to estimate advantage values for every step, and using these advantage estimates via RL enables us to address this problem. We show how the advantages can be used implicitly by preference optimization objectives. We show how training on an instance of this objective leads to $8 \times$ improvements in sample efficiency of the synthetic data used.

\textbf{Limitations.} While our work provides some results and conceptual models to understand the role of synthetic data for reasoning, there are still many open questions that need to be answered to fully understand its utility. While synthetic data from LLMs like Gemini and GPT-4 holds great potential, for more complex reasoning problems (more complicated than the datasets evaluated in our work),
synthetic data generated from more capable models can contain errors. Generating negative/positive data by referencing synthetic data answers can reinforce unwanted spurious correlations highlighted in our work. This means that novel recipes for generating synthetic problems may be utilized in the future, and our analysis might need to be re-done. That said, we believe that our insights about algorithmic behavior with synthetic data are still quite general and should transfer to these novel settings as well. Ultimately, we would want that training on synthetic data improves transfer and generalization abilities of the model in general reasoning scenarios, and to this end, an evaluation of transfer capabilities is an important avenue that future work should focus on.

\vspace{-0.2cm}
\section*{Acknowledgements}
\vspace{-0.15cm}

This work was done at CMU. We thank Vaishnavh Nagarajan, Yi Su, Aleksandra Faust, Christina Baek, Charlie Snell, Seohong Park,  Gaurav Ghosal, Aditi Raghunathan, Katie Kang, Don Dennis, Dhruv Malik, and Pratiksha Thaker for informative discussions and feedback on an earlier version of this paper. This work was supported by compute donations from Google Cloud (TRC) and MultiOn. AS thanks OpenAI and Google respectively for providing GPT4-Turbo and Gemini-1.5 Pro credits for academic use. AK and YG thank Tianhe Yu for feedback on the paper.

\bibliography{main}
\newpage
\appendix
\part*{Appendices}

\section{Proof of Theorem~\ref{thm:equivalence}}
\label{sec:proof-ofequiv}
\vspace{-0.2cm}

We first restate the theorem statement and then provide a proof for this below. Our main goal in this theorem is to show that training with per-step DPO is equivalent to running advantage-weighted RL shown in the theoretical result.

\begin{theorem}[Equivalence of advantage-weighted RL and DPO with per-step pairs]
\label{thm:equivalence_restated}
    The optimal policy from Equation~\ref{eq:dpo_objective} with $\dpodata$ given by $(\bx, [\by_{1:i}, \posby_{i+1}], [\by_{1:i}, \negby_{i+1}])$ where the positive and negative traces share prefix $\by_{1:i} \sim \pi_\sft$, and $\negby_{i+1}\sim \pi_\sft(\cdot | \bx, \by_{1:i})$, $\posby_{i+1} \sim \sigma(A_{\tilde{\pi}}(\bx, \by_{1:i}; \cdot)- A_{\tilde{\pi}}(\bx, \by_{1:i}; \negby_{i+1}))$,  
    is identical to the optima of the advantage-weighted RL objective:
    {
    \begin{align}
        \max_{\pi}~~& \mathbb{E}_{\bx \sim p_\mathrm{syn}(\bx), \by  \sim \pi_\sft(\cdot|\bx)} \left[ \sum_{i=1}^L \log \pi(\by_{i} \big| \bx, \by_{0:i-1}) \cdot \exp \left( A_{\tilde{\pi}}(\bx, \by_{0:i-1}, \by_{i}) / \beta \right) \right]. \vspace{-0.5cm}
        \label{eq:advanage_weighted_objective_appendix}
    \end{align}
    }
\end{theorem}
\begin{proof}
    To prove this statement, we make the following observation: DPO~\citep{rafailov2023direct} is equivalent to optimizing a KL-divergence penalized expected reward objective in an induced Bradly-Terry model of preferences defined by the reward function. That is, for any reward function $r(\bx, \by)$ over contexts $\bx \sim \mu$ and responses $\by$, the optimal solution to the following RL objective:
    \begin{align}
        \max_{\pi}~~ \mathbb{E}_{\bx \sim \mu, \by \sim \pi(\cdot|\bx)} \left[r(\bx, \by) \right] - \beta D_\mathrm{KL}(\pi(\cdot|\bx) || \pi_\sft(\cdot|\bx)),
    \end{align}
    is given by the following advantage-weighted optimal policy, $\pi^*(\cdot|\cdot)$:
    \begin{align}
        \label{eq:advantage_weighted_policy}
        \forall \bx, \by, ~~ \pi^*(\by|\bx) \propto \pi_\sft(\by|\bx) \cdot \exp \left(\frac{r(\bx, \by)}{\beta} \right),
    \end{align}
      and one can learn this optimal policy by running DPO on preference tuples $(\bx, \by_1, \by_2)$ sampled by the Bradly-Terry model~\citep{bradley1952rank} induced by the reward function $r$:
    \begin{align}
        \label{eq:bradly_terry}
        p \left( \by_1 \succcurlyeq \by_2 | \bx \right) = \frac{\exp(r(\bx, \by_1))}{\exp(r(\bx, \by_1)) + \exp(r(\bx, \by_2))}.
    \end{align}
    Given this background information, we know that the optimal advantage-weighted RL policy optimizing Equation~\ref{eq:advanage_weighted_objective_appendix} is given by: 
    \begin{align}
    \label{eq:optimal_adv_policy}
        \forall \bx, \by_{0:i}, ~~ \pi(\by_{i} | \bx, \by_{0:i-1}) \propto \pi_\sft(\by_i | \bx, \by_{0:i-1}) \cdot \exp \left( \frac{A_{\tilde{\pi}}(\bx, \by_{0:i-1}, \by_{i})}{\beta} \right).
    \end{align}
    Combining Equation~\ref{eq:optimal_adv_policy} with the equivalence between Equation~\ref{eq:advantage_weighted_policy} and the Bradly-Terry model (Equation~\ref{eq:bradly_terry}), we obtain that, if preference pairs $(\bx, [\by_{1:i}, \posby_{i+1}], [\by_{1:i}, \negby_{i+1}])$ were sampled from the SFT policy: $\posby_{i+1} \sim \pi_\sft(\cdot|\bx, \by_{0:i})$ and $\negby_{i+1} \sim \pi_\sft(\cdot|\bx, \by_{0:i})$, and labeled according to Equation~\ref{eq:bradly_terry} applied on advantage estimates, then we obtain the desired equivalence result. 
\end{proof}

\vspace{-0.35cm}
\section{Theory: Why Does Negative Data Improve Generalization?}
\label{sec:formal_statement}
\vspace{-0.2cm}

We saw in Section~\ref{sec:generalization} that collecting negative data from an appropriate SFT policy $\pi_\sft$ and an appropriate $K$, and training on this data improves generalization performance of the resulting model. In this section, building on the equivalence to advantage-weighted RL (Theorem~\ref{thm:equivalence}), we attempt to formalize this observation into a performance guarantee. In particular, we show below that training on negative data implies that we are able to improve over the SFT policy, especially via the detection of critical steps, that attain high advantages, $A_{\tilde{\pi}}(\bx, \by_{0:i-1}, \by_{i})$, that are otherwise not prioritized by training on positive data alone. Our theoretical result extends guarantees from the RL literature~\citep{kumar2022should} comparing RL with imitation learning to show that indeed the use of RL (and hence negative data) improves over imitation alone.

\textbf{Notation and setup.} Define the policy obtained after advantage-weighted RL training as $\pi_\text{neg}$. Concretely, $\pi_\negtext(\by|\bx)$ is given as:
\begin{align}
\label{eq:empirical_advantage}
    \!\!\!\!\!\!\!\!\!\!\forall \bx, \by_{0:j+1}, ~~ \pi_\negtext(\by_{j+1}|\bx, \by_{0:j}) = \frac{1}{\widehat{\mathbb{Z}}(\bx, \by_{0:j})} \pi_\sft(\by_{j+1}|\bx, \by_{0:j}) \cdot \exp \left(\frac{\hat{A}_{\tilde{\pi}}(\bx, \by_{0:j}, \by_{j+1})}{\beta} \right),
\end{align}
where the normalization factoris given by $\mathbb{Z}(\bx, \by_{0:j})$ for each of the per-step policy distributions. This normalization factor is a critical factor that will drive the core of the theoretical result. We also note that the normalization factor in Equation~\ref{eq:empirical_advantage} is derived from \emph{empirical} advantage estimates and not from the expected estimates for the advantage value.
Following \citet{agarwal2019reinforcement,kumar2022should}, we operate in a tabular setting with a discrete (but combinatorially-large and variable-length) action space of responses, and our proof follows Theorem 4.4 in \citet{kumar2022should}.

\begin{theorem}[Utility of negative data over positive data.]
\label{thm:neg_data_theorem}
    Let $\pi_\text{neg}$ denote the policy obtained after advantage-weighted RL (Equation~\ref{eq:empirical_advantage}) under an empirical distribution $\hat{\mu}$ over prompts $\bx$. Then the difference between the expected reward (i.e., task success rate), $J(\cdot)$, attained by $\pi_\text{neg}$ and $\pi_\sft$ is lower-bounded as:
\begin{align*}
    J(\pi_\text{neg}) - J(\pi_\sft) 
    \gtrsim & ~~~ \beta \cdot \mathbb{E}_{\bx_i \sim \hat{\mu}, \by_{i, 0:L} \sim \pi_\text{neg}(\cdot|\bx_{i})} \left[ \sum_{j=1}^L \log \mathbb{Z}(\bx_i, \by_{i, 0:j}) \right] \\ 
     & ~~~~~~~~~~~~~ - \text{(overestimation in } \hat{A}_{\tilde{\pi}}(\bx, \by_{0:i-1}, \by_{i})) + \frac{c_0}{\sqrt{|\syndata|}},  
\end{align*}
where $\mathbb{Z}(\clubsuit, \circ)$ denotes the sum over exponentiated differences of the advantage and log likelihood values under $\pi_\sft$ for all possible candidate steps given a problem $\clubsuit$ and a partial solution $\circ$. That is, 
\begin{align*}
    \mathbb{Z}(\clubsuit, \circ) := \sum_{\spadesuit \in ~\text{step candidates}} \exp \left( \frac{{A}_{\tilde{\pi}}(\clubsuit, \circ; \spadesuit)}{\beta} + \log \pi_\sft(\spadesuit | \clubsuit, \circ) \right),
\end{align*}
$c_0$ is a constant depending upon the Rademacher complexity of the space of policies $\pi_\text{neg}$ close to the SFT policy under the KL-divergence, $|\syndata|$ denotes the size of synthetic training prompts. 
\end{theorem}

\begin{proof}
To begin the proof, we recall that we are operating in a discrete action space of steps $\by_i$, although this space is exponentially large. Since we operate in discrete action spaces, we invoke Lemma 5 from \citet{agarwal2019reinforcement} for analyzing softmax policy gradient methods (this Lemma was extended by Lemma B.11 from \citet{kumar2022should} for comparing BC vs offline RL). Denote by $\widehat{J}(\pi)$, the reward attained by policy $\pi$ in expectation over the empirical distribution $\widehat{\mu}$:
\begin{align}
    \widehat{J}(\pi_\text{neg}) - \widehat{J}(\pi_\sft) := \mathbb{E}_{\bx \sim \widehat{\mu}}\left[\widehat{V}^{\pi_\text{neg}}(\bx)\right] -  \mathbb{E}_{\bx \sim \widehat{\mu}}\left[\widehat{V}^{\pi_\sft}(\bx)\right] \geq \beta \mathbb{E}_{\bx \sim \widehat{\mu}} \left[ \log \widehat{\mathbb{Z}}(\bx) \right].
\end{align}
We utilize the performance difference lemma~\citep{kakade2002approximately} on the MDP induced by the set of initial problems in the training distribution $\widehat{\mu}$, and the model induced deterministic dynamics distribution:
\begin{align*}
\widehat{J}(\pi_\text{neg}) &- \widehat{J}(\pi_\sft) = \sum_{j=1}^L \mathbb{E}_{\bx \sim \widehat{\mu}, \by_{0:j-1} \sim \pi_\text{neg}(\cdot|\bx)}\left[\sum_{\by_j} \pi_\text{neg}(\by_j|\bx, \by_{0:j-1}) \hat{A}_{\tilde{\pi}}(\bx, \by_{0:i-1}, \by_{i}) \right]\\
&= \sum_{j=1}^L \mathbb{E}_{\bx \sim \widehat{\mu}, \by_{0:j-1} \sim \pi_\text{neg}(\cdot|\bx)}\left[\sum_{\by_j} \pi_\text{neg}(\by_j|\bx, \by_{0:j-1}) \log \frac{\pi_\text{neg}(\by_j|\bx, \by_{0:j-1}) \cdot \widehat{\mathbb{Z}}(\bx, \by_{0:j})}{\pi_\sft(\by_j|\bx, \by_{0:j-1})} \right] \\
&= \beta \cdot \sum_{j=1}^L \mathbb{E}_{\bx \sim \widehat{\mu}, \by_{0:j-1} \sim \pi_\text{neg}(\cdot|\bx)}\left[ D_\mathrm{KL} \left( \pi_\text{neg}(\cdot|\bx, \by_{0:j-1}), \pi_\sft(\cdot|\bx, \by_{0:j-1}) \right) +  \log \widehat{\mathbb{Z}}(\bx, \by_{0:j}) \right] \\
& \geq \beta \cdot \sum_{j=1}^L \mathbb{E}_{\bx \sim \widehat{\mu}, \by_{0:j-1} \sim \pi_\text{neg}(\cdot|\bx)} \left[ \log \widehat{\mathbb{Z}}(\bx, \by_{0:j}) \right].\\
& = \beta \cdot \mathbb{E}_{\bx \sim \hat{\mu}, \by_{i, 0:L} \sim \pi_\text{neg}(\cdot|\bx)} \left[ \sum_{j=1}^L \log \mathbb{Z}(\bx, \by_{0:j}) \right].
\end{align*}
Now, we can prove the desired result by accounting for the gap in the success rate between the actual distribution over $\bx \sim \mu$ and the empirical distribution induced by problems in the dataset $\widehat{\mu}$:
\begin{align*}
    J(\pi_\text{neg}) - J(\pi_\sft) &\geq \underbrace{J(\pi_\text{neg}) - \widehat{J}(\pi_\text{neg})}_{\text{(a)}} + \underbrace{\widehat{J}(\pi_\text{neg}) - \widehat{J}(\pi_\sft)}_{\text{(b)}} - \underbrace{J(\pi_\sft) - \widehat{J}(\pi_\sft)}_{\text{(c)}}\\
    &\geq \beta \cdot \mathbb{E}_{\bx \sim \hat{\mu}, \by_{i, 0:L} \sim \pi_\text{neg}(\cdot|\bx)} \left[ \sum_{j=1}^L \log \widehat{\mathbb{Z}}(\bx, \by_{0:j}) \right] - \frac{c_0}{\sqrt{|\syndata|}}\\
    & \geq \beta \cdot \mathbb{E}_{\bx \sim \hat{\mu}, \by_{i, 0:L} \sim \pi_\text{neg}(\cdot|\bx)} \left[ \sum_{j=1}^L \log {\mathbb{Z}}(\bx, \by_{0:j}) \right] - \frac{c_0}{\sqrt{|\syndata|}} + \Delta,
\end{align*}
where $c_0$ is a constant that depends on the Rademacher complexity of the function class of policies $\pi_\text{neg}$ (coming from a uniform bound that we invokve to bound term \textit{(a)}, since $\pi_\negtext$ depends on the dataset samples), and this term arises since the empirical distribution over prompts is not the same as the true population. This concentration term decays as the size of the synthetic data (number of problems) are scaled up. The term $\Delta$ denotes the overestimation error between the estimated advantages $\widehat{A}_{\tilde{\pi}}(\bx, \by_{0:i-1}, \by_i)$ and the true advantages ${A}_{\tilde{\pi}}(\bx, \by_{0:i-1}, \by_i)$, in expectation under the distribution of the learned policy. The estimation error $\Delta$ depends on $\pi_\sft$ and the value of $K$ used if the rollout policy $\tilde{\pi}$ corresponds to the BoK($\pi_\sft$) policy. This proves the theorem.
\end{proof}

\textbf{Interpretation \& perspectives.} Also note that the improvement in performance between $\pi_\negtext$ and $\pi_\sft$ depends on the advantage estimate: if the advantage estimates are high, then this term is large, meaning that the more the fraction of high-advantage critical states, the higher the improvement. In addition, the bound also says that if the over-estimation $\Delta$ in the advantage estimate is large, the performance improvement is small. This is perhaps expected: consider the scenario when the BoK($\pi_\sft$) policy is used to collect data, for a large $K$. In this scenario, the divergence between the empirical advantage estimate $\widehat{A}_{\tilde{\pi}}$ and the expected estimate $A_{\tilde{\pi}}$ is likely large. In the worst case, the estimate $\widehat{A}_{\tilde{\pi}}$ can arbitrarily overestimate $A_{\tilde{\pi}}$, as it would take on a high value even if there is just \emph{one} sequence among the $K$ rollouts that successfully solves the problem. For example, a spurious step may be labeled incorrectly as critical in this case and training on negative data may not improve (consistent with running per-step DPO on an over-trained SFT checkpoint in Figure~\ref{fig:dpo_didactic}). On the other hand, if advantages are more accurate, training on negative data should improve performance.

\vspace{-0.3cm}
\section{Additional Related Work}
\label{sec:additional-relwork}
\vspace{-0.3cm}

\textbf{Failure modes for supervised finetuning (SFT).}
First, since SFT induces an open-loop~\citep{wu2023pairwise} next-token prediction loss, prediction errors on even a single token can snowball during inference, leading to poor performance on the prompts appearing in the dataitself~\cite{kaariainen2006lower,ross2010efficient}. Second, even when an LLM has perfectly cloned the SFT data, it is prone to memorize ``hard to learn'' tokens~\cite{tirumala2022memorization}, especially in planning and lookahead tasks~\cite{mccoy2023embers,momennejad2024evaluating}, which is critical for math reasoning.
This leads to poor generalization~\cite{bachmann2024pitfalls,dziri2024faith} and hallucination on new novel, test-tim prompts~\cite{kang2024unfamiliar}. In this work, we study how synthetic data methods can address these failures via: (i) maximizing likelihood on positive data generated from both the SFT policy and a stronger teacher that enjoys improved coverage over new states, 
and (ii) preference optimization using the negative data generated from the SFT policy.

\section{Synthetic Data Generation}
\label{sec:synth_data_appendix}
\vspace{-0.3cm}

\begin{tcolorbox}[colback=blue!5!white, colframe=blue!75!black, title= Prompt used for GSM8K/MATH synthetic data~\citep{li2024common}]
Please act as a professional math teacher.
Your goal is to create high quality math problems to help students learn math.
You will be given a math question. Please create a new question based on the Given Question and following
instructions.

To achieve the goal, you have one job.

\# Please generate a similar but new question according to the Given Question.

You have four principles to do this.
\# Ensure the new question only asks for one thing, be reasonable, be based on the Given Question, and can be
answered with only a number(float or integer). For example, DO NOT ask, ‘what is the amount of A, B and
C?’.

\# Ensure the new question is in line with common sense of life. For example, the amount someone has or pays
must be a positive number, and the number of people must be an integer.

\# Ensure your student can answer the new question without the given question. If you want to use some
numbers, conditions or background in the given question, please restate them to ensure no information is
omitted in your new question.

\# You only need to create the new question. Please DO NOT solve it.

Given Question: <insert question from original dataset here>

Your output should be in the following format:

CREATED QUESTION: <your created question>

\end{tcolorbox}

For GSM8K, we replace the phrase ``Your goal is to create high quality math problems to help students learn math.'' with ``Your goal is to create high quality math \emph{word} problems to help students learn math.'', as we found this to produce problems that were closer to GSM-style problems.

To generate the synthetic data, we used OpenAI credits worth approximately 3000 US dollars.

\section{Negative Data Identifies Spurious Steps with Advantage Estimates}
\label{app:spurious-examples}
\vspace{-0.3cm}

In this section, we present several examples from the GSM8K and MATH synthetic datasets  of model-generated (from $\pi_\sft$) positive data with spurious steps.  Then,  we show how computing advantage values accurately identify the spurious steps.  We highlight the steps that were deemed to be low or high advantage based on the estimated advantage value. The darker the intensity of the color, the lower the advantage of that step. With this qualitative probe, we see that the steps that are darkest are indeed either incorrect or correct but involve a reasoning step that is irrelevant for solving the problem at hand.

\begin{tcolorbox}[colback=blue!5!white, colframe=blue!75!black, title=Spurious step identified by negative data]
\textbf{Question:} 
Ken created a care package to send to his brother, who was away at boarding school.  Ken placed a box on a scale, and then he poured into the box enough jelly beans to bring the weight to 2 pounds.  Then, he added enough brownies to cause the weight to triple.  Next, he added another 2 pounds of jelly beans.  And finally, he added enough gummy worms to double the weight once again.  What was the final weight of the box of goodies, in pounds?

\vspace{2em}
\textbf{Positive response with identified spurious step: }

{\setlength{\fboxsep}{0pt}\colorbox{blue!5!white}{\parbox{\textwidth}{
\colorbox{red!32.81}{\strut When} \colorbox{red!32.81}{\strut the} \colorbox{red!32.81}{\strut weight} \colorbox{red!32.81}{\strut was} \colorbox{red!32.81}{\strut originally} \colorbox{red!32.81}{\strut on} \colorbox{red!32.81}{\strut the} \colorbox{red!32.81}{\strut scale,} \colorbox{red!32.81}{\strut the} \colorbox{red!32.81}{\strut box} \colorbox{red!32.81}{\strut weighed} \colorbox{red!32.81}{\strut 0} \colorbox{red!32.81}{\strut pounds.} \colorbox{red!7.36}{\strut After} \colorbox{red!7.36}{\strut Ken} \colorbox{red!7.36}{\strut added} \colorbox{red!7.36}{\strut the} \colorbox{red!7.36}{\strut jelly} \colorbox{red!7.36}{\strut beans} \colorbox{red!7.36}{\strut to} \colorbox{red!7.36}{\strut the} \colorbox{red!7.36}{\strut box,} \colorbox{red!7.36}{\strut its} \colorbox{red!7.36}{\strut weight} \colorbox{red!7.36}{\strut was} \colorbox{red!7.36}{\strut 2} \colorbox{red!7.36}{\strut pounds} \colorbox{red!7.36}{\strut -} \colorbox{red!7.36}{\strut 0} \colorbox{red!7.36}{\strut pounds} \colorbox{red!7.36}{\strut =} \colorbox{red!7.36}{\strut <<2-0=2>>2} \colorbox{red!7.36}{\strut pounds.} \colorbox{red!21.04}{\strut Then} \colorbox{red!21.04}{\strut Ken} \colorbox{red!21.04}{\strut added} \colorbox{red!21.04}{\strut the} \colorbox{red!21.04}{\strut brownies} \colorbox{red!21.04}{\strut to} \colorbox{red!21.04}{\strut the} \colorbox{red!21.04}{\strut box,} \colorbox{red!21.04}{\strut which} \colorbox{red!21.04}{\strut caused} \colorbox{red!21.04}{\strut the} \colorbox{red!21.04}{\strut box} \colorbox{red!21.04}{\strut to} \colorbox{red!21.04}{\strut weigh} \colorbox{red!21.04}{\strut 2} \colorbox{red!21.04}{\strut pounds} \colorbox{red!21.04}{\strut *} \colorbox{red!21.04}{\strut 3} \colorbox{red!21.04}{\strut =} \colorbox{red!21.04}{\strut <<2*3=6>>6} \colorbox{red!21.04}{\strut pounds.} \colorbox{red!29.78}{\strut The} \colorbox{red!29.78}{\strut next} \colorbox{red!29.78}{\strut step} \colorbox{red!29.78}{\strut was} \colorbox{red!29.78}{\strut to} \colorbox{red!29.78}{\strut add} \colorbox{red!29.78}{\strut 2} \colorbox{red!29.78}{\strut pounds} \colorbox{red!29.78}{\strut of} \colorbox{red!29.78}{\strut jelly} \colorbox{red!29.78}{\strut beans,} \colorbox{red!29.78}{\strut which} \colorbox{red!29.78}{\strut caused} \colorbox{red!29.78}{\strut the} \colorbox{red!29.78}{\strut box} \colorbox{red!29.78}{\strut to} \colorbox{red!29.78}{\strut weigh} \colorbox{red!29.78}{\strut 6} \colorbox{red!29.78}{\strut pounds} \colorbox{red!29.78}{\strut +} \colorbox{red!29.78}{\strut 2} \colorbox{red!29.78}{\strut pounds} \colorbox{red!29.78}{\strut =} \colorbox{red!29.78}{\strut <<6+2=8>>8} \colorbox{red!29.78}{\strut pounds.} \colorbox{red!61.99}{\strut First,} \colorbox{red!61.99}{\strut we} \colorbox{red!61.99}{\strut need} \colorbox{red!61.99}{\strut to} \colorbox{red!61.99}{\strut figure} \colorbox{red!61.99}{\strut out} \colorbox{red!61.99}{\strut what} \colorbox{red!61.99}{\strut the} \colorbox{red!61.99}{\strut weight} \colorbox{red!61.99}{\strut of} \colorbox{red!61.99}{\strut the} \colorbox{red!61.99}{\strut box} \colorbox{red!61.99}{\strut was} \colorbox{red!61.99}{\strut when} \colorbox{red!61.99}{\strut Ken} \colorbox{red!61.99}{\strut put} \colorbox{red!61.99}{\strut the} \colorbox{red!61.99}{\strut jelly} \colorbox{red!61.99}{\strut beans} \colorbox{red!61.99}{\strut in} \colorbox{red!61.99}{\strut it.} \colorbox{red!61.99}{\strut 2} \colorbox{red!61.99}{\strut pounds} \colorbox{red!61.99}{\strut of} \colorbox{red!61.99}{\strut jelly} \colorbox{red!61.99}{\strut beans} \colorbox{red!61.99}{\strut was} \colorbox{red!61.99}{\strut enough} \colorbox{red!61.99}{\strut to} \colorbox{red!61.99}{\strut triple} \colorbox{red!61.99}{\strut the} \colorbox{red!61.99}{\strut weight,} \colorbox{red!61.99}{\strut so} \colorbox{red!61.99}{\strut the} \colorbox{red!61.99}{\strut box} \colorbox{red!61.99}{\strut was} \colorbox{red!61.99}{\strut at} \colorbox{red!61.99}{\strut 2} \colorbox{red!61.99}{\strut pounds} \colorbox{red!61.99}{\strut /} \colorbox{red!61.99}{\strut 3} \colorbox{red!61.99}{\strut =} \colorbox{red!61.99}{\strut <<2/3=.6666666666666667>>.6666666666666667} \colorbox{red!61.99}{\strut pounds.} \colorbox{red!30.33}{\strut Finally,} \colorbox{red!30.33}{\strut Ken} \colorbox{red!30.33}{\strut added} \colorbox{red!30.33}{\strut the} \colorbox{red!30.33}{\strut gummy} \colorbox{red!30.33}{\strut worms,} \colorbox{red!30.33}{\strut causing} \colorbox{red!30.33}{\strut the} \colorbox{red!30.33}{\strut box} \colorbox{red!30.33}{\strut to} \colorbox{red!30.33}{\strut weigh} \colorbox{red!30.33}{\strut 8} \colorbox{red!30.33}{\strut pounds} \colorbox{red!30.33}{\strut *} \colorbox{red!30.33}{\strut 2} \colorbox{red!30.33}{\strut =} \colorbox{red!30.33}{\strut <<8*2=16>>16} \colorbox{red!30.33}{\strut pounds.} \colorbox{red!4.35}{\strut The} \colorbox{red!4.35}{\strut answer} \colorbox{red!4.35}{\strut is} \colorbox{red!4.35}{\strut 16} 
}}}

\end{tcolorbox}

\begin{tcolorbox}[colback=blue!5!white, colframe=blue!75!black, title=Spurious step identified by negative data]
\textbf{Question:} 
Betty is saving money for a new wallet which costs \$100. Betty has only half of the money she needs. Her parents decided to give her \$15 for that purpose, and her grandparents twice as much as her parents. How much more money does Betty need to buy the wallet?

\vspace{2em}
\textbf{Positive response with identified spurious step: }

{\setlength{\fboxsep}{0pt}\colorbox{blue!5!white}{\parbox{\textwidth}{
\colorbox{red!19.9}{\strut Betty's} \colorbox{red!19.9}{\strut parents} \colorbox{red!19.9}{\strut gave} \colorbox{red!19.9}{\strut her} \colorbox{red!19.9}{\strut 15,} \colorbox{red!19.9}{\strut so} \colorbox{red!19.9}{\strut her} \colorbox{red!19.9}{\strut grandparents} \colorbox{red!19.9}{\strut gave} \colorbox{red!19.9}{\strut her} \colorbox{red!19.9}{\strut 2} \colorbox{red!19.9}{\strut *} \colorbox{red!19.9}{\strut 15} \colorbox{red!19.9}{\strut =} \colorbox{red!19.9}{\strut <<15*2=30>>30.} \colorbox{red!16.94}{\strut In} \colorbox{red!16.94}{\strut total,} \colorbox{red!16.94}{\strut Betty} \colorbox{red!16.94}{\strut received} \colorbox{red!16.94}{\strut 30} \colorbox{red!16.94}{\strut +} \colorbox{red!16.94}{\strut 15} \colorbox{red!16.94}{\strut =} \colorbox{red!16.94}{\strut <<30+15=45>>45.} \colorbox{red!62.96}{\strut The} \colorbox{red!62.96}{\strut total} \colorbox{red!62.96}{\strut amount} \colorbox{red!62.96}{\strut of} \colorbox{red!62.96}{\strut money} \colorbox{red!62.96}{\strut Betty} \colorbox{red!62.96}{\strut needs} \colorbox{red!62.96}{\strut is} \colorbox{red!62.96}{\strut 100} \colorbox{red!62.96}{\strut *} \colorbox{red!62.96}{\strut 2} \colorbox{red!62.96}{\strut =} \colorbox{red!62.96}{\strut <<100*2=200>>200.} \colorbox{red!23.31}{\strut Betty} \colorbox{red!23.31}{\strut needs} \colorbox{red!23.31}{\strut 100} \colorbox{red!23.31}{\strut /} \colorbox{red!23.31}{\strut 2} \colorbox{red!23.31}{\strut =} \colorbox{red!23.31}{\strut <<100/2=50>>50} \colorbox{red!23.31}{\strut in} \colorbox{red!23.31}{\strut total.} \colorbox{red!37.06}{\strut Betty} \colorbox{red!37.06}{\strut still} \colorbox{red!37.06}{\strut needs} \colorbox{red!37.06}{\strut to} \colorbox{red!37.06}{\strut save} \colorbox{red!37.06}{\strut 50} \colorbox{red!37.06}{\strut -} \colorbox{red!37.06}{\strut 45} \colorbox{red!37.06}{\strut =} \colorbox{red!37.06}{\strut <<50-45=5>>5.} \colorbox{red!24.35}{\strut The} \colorbox{red!24.35}{\strut answer} \colorbox{red!24.35}{\strut is} \colorbox{red!24.35}{\strut 5} 
}}}
\end{tcolorbox}

\begin{tcolorbox}[colback=blue!5!white, colframe=blue!75!black, title=Spurious step identified by negative data]
\textbf{Question:} 
To make pizza, together with other ingredients, Kimber needs 10 cups of water, 16 cups of flour, and 1/2 times as many teaspoons of salt as the number of cups of flour. Calculate the combined total number of cups of water, flour, and teaspoons of salt that she needs to make the pizza.

\vspace{1em}
\textbf{Positive response with identified spurious step:} 

{\setlength{\fboxsep}{0pt}\colorbox{blue!5!white}{\parbox{\textwidth}{
\colorbox{red!23.08}{\strut Since} \colorbox{red!23.08}{\strut 1} \colorbox{red!23.08}{\strut cup} \colorbox{red!23.08}{\strut is} \colorbox{red!23.08}{\strut equal} \colorbox{red!23.08}{\strut to} \colorbox{red!23.08}{\strut 16} \colorbox{red!23.08}{\strut teaspoons,} \colorbox{red!23.08}{\strut then} \colorbox{red!23.08}{\strut Kimber} \colorbox{red!23.08}{\strut needs} \colorbox{red!23.08}{\strut 128/16} \colorbox{red!23.08}{\strut =} \colorbox{red!23.08}{\strut <<128/16=8>>8} \colorbox{red!23.08}{\strut cups} \colorbox{red!23.08}{\strut of} \colorbox{red!23.08}{\strut salt.} \colorbox{red!23.92}{\strut So} \colorbox{red!23.92}{\strut in} \colorbox{red!23.92}{\strut total,} \colorbox{red!23.92}{\strut she} \colorbox{red!23.92}{\strut needs} \colorbox{red!23.92}{\strut 10+16+8} \colorbox{red!23.92}{\strut =} \colorbox{red!23.92}{\strut <<10+16+8=34>>34} \colorbox{red!23.92}{\strut cups} \colorbox{red!23.92}{\strut of} \colorbox{red!23.92}{\strut water,} \colorbox{red!23.92}{\strut flour,} \colorbox{red!23.92}{\strut and} \colorbox{red!23.92}{\strut salt.} \colorbox{red!79.49}{\strut In} \colorbox{red!79.49}{\strut cups,} \colorbox{red!79.49}{\strut the} \colorbox{red!79.49}{\strut amount} \colorbox{red!79.49}{\strut of} \colorbox{red!79.49}{\strut salt} \colorbox{red!79.49}{\strut is} \colorbox{red!79.49}{\strut 128/6} \colorbox{red!79.49}{\strut =} \colorbox{red!79.49}{\strut <<128/6=21.33>>21.33} \colorbox{red!79.49}{\strut cups} \colorbox{red!10.11}{\strut The} \colorbox{red!10.11}{\strut answer} \colorbox{red!10.11}{\strut is} \colorbox{red!10.11}{\strut 34} 
}}}

\end{tcolorbox}

\begin{tcolorbox}[colback=blue!5!white, colframe=blue!75!black, title=Spurious step identified by negative data]
\textbf{Question:} 
Eight friends compete in a pie eating contest. The chart displays the number of pies each person consumed. Emily ate the most pies, and Mike ate the least. How many more pies did Emily eat compared to Mike?
\begin{verbatim}
    [asy]\ndefaultpen(linewidth(1pt)+fontsize(10pt));
    \npair[] yaxis = new pair[10];\nfor( int i = 0 ; i < 10 ; ++i )
    {\n\nyaxis[i] = (0,i);\n\ndraw(yaxis[i]--yaxis[i]+(17,0));\n}
    \ndraw((0,0)--(0,9));\ndraw((17,9)--(17,0));
    \n\nfill((1,0)--(1,5)--(2,5)--(2,0)--cycle,grey);
    \nfill((3,0)--(3,8)--(4,8)--(4,0)--cycle,grey);
    \nfill((5,0)--(5,3)--(6,3)--(6,0)--cycle,grey);
    \n\nlabel("0",yaxis[0],W);
    \nlabel("1",yaxis[1],W);\nlabel("2",yaxis[2],W);
    \nlabel("3",yaxis[3],W);\nlabel("4",yaxis[4],W);
    \nlabel("5",yaxis[5],W);\nlabel("6",yaxis[6],W);
    \nlabel("7",yaxis[7],W);\nlabel("8",yaxis[8],W);
    \nlabel("Friends/Participants",(8.5,0),S);
    \nlabel("Results of a Pie Eating Contest",(8.5,9),N);
    \nlabel(rotate(90)*"$\\#$ of Pies Eaten",(-1,4.5),W);[/asy]
\end{verbatim}

\textbf{Positive response with identified spurious step: }

{\setlength{\fboxsep}{0pt}\colorbox{blue!5!white}{\parbox{\textwidth}{
\colorbox{red!22.12}{\strut To} \colorbox{red!22.12}{\strut solve} \colorbox{red!22.12}{\strut this} \colorbox{red!22.12}{\strut problem,} \colorbox{red!22.12}{\strut we} \colorbox{red!22.12}{\strut need} \colorbox{red!22.12}{\strut to} \colorbox{red!22.12}{\strut identify} \colorbox{red!22.12}{\strut the} \colorbox{red!22.12}{\strut number} \colorbox{red!22.12}{\strut of} \colorbox{red!22.12}{\strut pies} \colorbox{red!22.12}{\strut Emily} \colorbox{red!22.12}{\strut and} \colorbox{red!22.12}{\strut Mike} \colorbox{red!22.12}{\strut consumed,} \colorbox{red!22.12}{\strut and} \colorbox{red!22.12}{\strut then} \colorbox{red!22.12}{\strut find} \colorbox{red!22.12}{\strut the} \colorbox{red!22.12}{\strut difference} \colorbox{red!22.12}{\strut between} \colorbox{red!22.12}{\strut their} \colorbox{red!22.12}{\strut consumption.} \colorbox{red!22.12}{\strut From} \colorbox{red!22.12}{\strut the} \colorbox{red!22.12}{\strut chart} \colorbox{red!22.12}{\strut provided:} \colorbox{red!38.7}{\strut -} \colorbox{red!38.7}{\strut Emily:} \colorbox{red!38.7}{\strut Consumed} \colorbox{red!38.7}{\strut 8} \colorbox{red!38.7}{\strut pies.} \colorbox{red!5.89}{\strut -} \colorbox{red!5.89}{\strut Mike:} \colorbox{red!5.89}{\strut Consumed} \colorbox{red!5.89}{\strut 1} \colorbox{red!5.89}{\strut pie.} \colorbox{red!75.28}{\strut To} \colorbox{red!75.28}{\strut determine} \colorbox{red!75.28}{\strut how} \colorbox{red!75.28}{\strut many} \colorbox{red!75.28}{\strut more} \colorbox{red!75.28}{\strut pies} \colorbox{red!75.28}{\strut Emily} \colorbox{red!75.28}{\strut ate} \colorbox{red!75.28}{\strut compared} \colorbox{red!75.28}{\strut to} \colorbox{red!75.28}{\strut Mike,} \colorbox{red!75.28}{\strut we} \colorbox{red!75.28}{\strut first} \colorbox{red!75.28}{\strut identify} \colorbox{red!75.28}{\strut how} \colorbox{red!75.28}{\strut many} \colorbox{red!75.28}{\strut pies} \colorbox{red!75.28}{\strut each} \colorbox{red!75.28}{\strut person} \colorbox{red!75.28}{\strut ate.} \colorbox{red!75.28}{\strut -} \colorbox{red!75.28}{\strut Emily} \colorbox{red!75.28}{\strut ate} \colorbox{red!75.28}{\strut 5} \colorbox{red!75.28}{\strut pies.} \colorbox{red!75.28}{\strut -} \colorbox{red!75.28}{\strut Mike} \colorbox{red!75.28}{\strut ate} \colorbox{red!75.28}{\strut 1} \colorbox{red!75.28}{\strut pie.} \colorbox{red!13.68}{\strut Now,} \colorbox{red!13.68}{\strut calculating} \colorbox{red!13.68}{\strut the} \colorbox{red!13.68}{\strut difference} \colorbox{red!13.68}{\strut between} \colorbox{red!13.68}{\strut Emily's} \colorbox{red!13.68}{\strut and} \colorbox{red!13.68}{\strut Mike's} \colorbox{red!13.68}{\strut consumption:} \colorbox{red!33.89}{\strut [} \colorbox{red!7.77}{\strut \text{Difference}} \colorbox{red!7.77}{\strut =} \colorbox{red!7.77}{\strut \text{Emily's} \colorbox{red!7.77}{\strut consumption}} \colorbox{red!7.77}{\strut -} \colorbox{red!7.77}{\strut \text{Mike's} \colorbox{red!7.77}{\strut consumption}} \colorbox{red!7.77}{\strut =} \colorbox{red!7.77}{\strut 8} \colorbox{red!7.77}{\strut -} \colorbox{red!7.77}{\strut 1} \colorbox{red!7.77}{\strut =} \colorbox{red!7.77}{\strut 7.} \colorbox{red!30.23}{\strut ]} \colorbox{red!4.3}{\strut Therefore,} \colorbox{red!4.3}{\strut Emily} \colorbox{red!4.3}{\strut ate} \colorbox{red!4.3}{\strut 7} \colorbox{red!4.3}{\strut more} \colorbox{red!4.3}{\strut pies} \colorbox{red!4.3}{\strut than} \colorbox{red!4.3}{\strut Mike.} \colorbox{red!24.76}{\strut The} \colorbox{red!24.76}{\strut answer} \colorbox{red!24.76}{\strut is} \colorbox{red!24.76}{\strut \boxed{7}} 
}}}

\end{tcolorbox}

\begin{tcolorbox}[colback=blue!5!white, colframe=blue!75!black, title=Spurious step identified by negative data]
\textbf{Question:} 
When the base-16 number $1AB3_{16}$ is written in base 2, how many base-2 digits (bits) does it have?

\textbf{Positive response with identified spurious step: }

{\setlength{\fboxsep}{0pt}\colorbox{blue!5!white}{\parbox{\textwidth}{
\colorbox{red!27.34}{\strut To} \colorbox{red!27.34}{\strut solve} \colorbox{red!27.34}{\strut this} \colorbox{red!27.34}{\strut problem,} \colorbox{red!27.34}{\strut we} \colorbox{red!27.34}{\strut first} \colorbox{red!27.34}{\strut need} \colorbox{red!27.34}{\strut to} \colorbox{red!27.34}{\strut convert} \colorbox{red!27.34}{\strut the} \colorbox{red!27.34}{\strut base-16} \colorbox{red!27.34}{\strut number} \colorbox{red!27.34}{\strut 1AB3\_\{16\}} \colorbox{red!27.34}{\strut to} \colorbox{red!27.34}{\strut base-2.} \colorbox{red!30.56}{\strut Step} \colorbox{red!30.56}{\strut 1:} \colorbox{red!30.56}{\strut Convert} \colorbox{red!30.56}{\strut each} \colorbox{red!30.56}{\strut digit} \colorbox{red!30.56}{\strut of} \colorbox{red!30.56}{\strut the} \colorbox{red!30.56}{\strut number} \colorbox{red!30.56}{\strut from} \colorbox{red!30.56}{\strut base-16} \colorbox{red!30.56}{\strut to} \colorbox{red!30.56}{\strut base-2.} \colorbox{red!7.6}{\strut The} \colorbox{red!7.6}{\strut number} \colorbox{red!7.6}{\strut 1AB3\_\{16\}} \colorbox{red!7.6}{\strut consists} \colorbox{red!7.6}{\strut of} \colorbox{red!7.6}{\strut the} \colorbox{red!7.6}{\strut digits} \colorbox{red!7.6}{\strut 1,} \colorbox{red!7.6}{\strut A,} \colorbox{red!7.6}{\strut B,} \colorbox{red!7.6}{\strut and} \colorbox{red!7.6}{\strut 3.} \colorbox{red!6.2}{\strut Step} \colorbox{red!6.2}{\strut 2:} \colorbox{red!6.2}{\strut Convert} \colorbox{red!6.2}{\strut each} \colorbox{red!6.2}{\strut digit} \colorbox{red!6.2}{\strut individually} \colorbox{red!6.2}{\strut to} \colorbox{red!6.2}{\strut base-2} \colorbox{red!6.2}{\strut representation.} \colorbox{red!18.02}{\strut -} \colorbox{red!18.02}{\strut 1\_\{16\}} \colorbox{red!18.02}{\strut =} \colorbox{red!18.02}{\strut 0001\_2} \colorbox{red!24.41}{\strut -} \colorbox{red!24.41}{\strut A\_\{16\}} \colorbox{red!24.41}{\strut =} \colorbox{red!24.41}{\strut 1010\_2} \colorbox{red!35.94}{\strut -} \colorbox{red!35.94}{\strut B\_\{16\}} \colorbox{red!35.94}{\strut =} \colorbox{red!35.94}{\strut 1011\_2} \colorbox{red!4.07}{\strut -} \colorbox{red!4.07}{\strut 3\_\{16\}} \colorbox{red!4.07}{\strut =} \colorbox{red!4.07}{\strut 0011\_2} \colorbox{red!33.31}{\strut Step} \colorbox{red!33.31}{\strut 3:} \colorbox{red!33.31}{\strut Combine} \colorbox{red!33.31}{\strut the} \colorbox{red!33.31}{\strut base-2} \colorbox{red!33.31}{\strut representations} \colorbox{red!33.31}{\strut of} \colorbox{red!33.31}{\strut each} \colorbox{red!33.31}{\strut digit.} \colorbox{red!16.65}{\strut Concatenating} \colorbox{red!16.65}{\strut the} \colorbox{red!16.65}{\strut individual} \colorbox{red!16.65}{\strut base-2} \colorbox{red!16.65}{\strut representations,} \colorbox{red!16.65}{\strut we} \colorbox{red!16.65}{\strut get:} \colorbox{red!37.28}{\strut 1AB3\_\{16\}} \colorbox{red!37.28}{\strut =} \colorbox{red!37.28}{\strut 0001101010110011\_2} \colorbox{red!27.48}{\strut Step} \colorbox{red!27.48}{\strut 4:} \colorbox{red!27.48}{\strut Count} \colorbox{red!27.48}{\strut the} \colorbox{red!27.48}{\strut total} \colorbox{red!27.48}{\strut number} \colorbox{red!27.48}{\strut of} \colorbox{red!27.48}{\strut digits} \colorbox{red!27.48}{\strut in} \colorbox{red!27.48}{\strut the} \colorbox{red!27.48}{\strut base-2} \colorbox{red!27.48}{\strut representation.} \colorbox{red!7.53}{\strut The} \colorbox{red!7.53}{\strut base-2} \colorbox{red!7.53}{\strut representation} \colorbox{red!7.53}{\strut of} \colorbox{red!7.53}{\strut 1AB3\_\{16\}} \colorbox{red!7.53}{\strut is:} \colorbox{red!37.24}{\strut 0001101010110011\_2} \colorbox{red!80.52}{\strut The} \colorbox{red!80.52}{\strut base-16} \colorbox{red!80.52}{\strut number} \colorbox{red!80.52}{\strut (1AB3\_\{16\}} \colorbox{red!80.52}{\strut )} \colorbox{red!80.52}{\strut can} \colorbox{red!80.52}{\strut be} \colorbox{red!80.52}{\strut written} \colorbox{red!80.52}{\strut in} \colorbox{red!80.52}{\strut base} \colorbox{red!80.52}{\strut 10} \colorbox{red!80.52}{\strut (decimal)} \colorbox{red!80.52}{\strut as:} \colorbox{red!80.52}{\strut [} \colorbox{red!80.52}{\strut 1} \colorbox{red!80.52}{\strut times} \colorbox{red!80.52}{\strut 16\^3} \colorbox{red!80.52}{\strut +} \colorbox{red!80.52}{\strut A} \colorbox{red!80.52}{\strut times} \colorbox{red!80.52}{\strut 16\^2} \colorbox{red!80.52}{\strut +} \colorbox{red!80.52}{\strut B} \colorbox{red!80.52}{\strut times} \colorbox{red!80.52}{\strut 16} \colorbox{red!80.52}{\strut +} \colorbox{red!80.52}{\strut 3} \colorbox{red!80.52}{\strut ]} \colorbox{red!80.52}{\strut where} \colorbox{red!80.52}{\strut (A} \colorbox{red!80.52}{\strut )} \colorbox{red!80.52}{\strut and} \colorbox{red!80.52}{\strut (B} \colorbox{red!80.52}{\strut )} \colorbox{red!80.52}{\strut are} \colorbox{red!80.52}{\strut decimal} \colorbox{red!80.52}{\strut digits.} \colorbox{red!11.52}{\strut This} \colorbox{red!11.52}{\strut sequence} \colorbox{red!11.52}{\strut of} \colorbox{red!11.52}{\strut digits} \colorbox{red!11.52}{\strut has} \colorbox{red!11.52}{\strut a} \colorbox{red!11.52}{\strut total} \colorbox{red!11.52}{\strut of} \colorbox{red!11.52}{\strut 16} \colorbox{red!11.52}{\strut digits.} \colorbox{red!24.13}{\strut The} \colorbox{red!24.13}{\strut answer} \colorbox{red!24.13}{\strut is} \colorbox{red!24.13}{\strut boxed\{16\}} 
}}}
\end{tcolorbox}

\section{Details on Star Graph Problem}
\label{sec:star_graph_appendix}

The star graph problem we study is borrowed from ~\citet{bachmann2024pitfalls}, where given a graph in the shape of a star and a query (center/end node pair), the model is asked to output the full path between the start/end nodes. 

\textbf{Goal.} \citet{bachmann2024pitfalls} show that $\pi_\sft$ minimizes SFT loss by memorizing the ``hard-to-predict'' node adjacent to the center, and copying the rest of the path from the input graph. This task is highlights the failure of SFT at planning problems (akin to math reasoning). 
 Thus, we use this as a case study to understand:

\begin{itemize}[leftmargin=2em]
    \item  when accurate advantage estimation is possible with few negative samples from the $\pi_\sft$ model.
    \item whether there are generalization benefits of advantage-weighted RL when advantage estimates are accurate
    \item when advantage-weighted RL can unlearn the memorized feature that causes $\pi_\sft$ to fail. 
\end{itemize}

\textbf{SFT dataset.} The data we use for supervised fine-tuning consists of 30000 of random star graphs (see examples below) where each graph has a centre node with out degree 2. Hence, there are two paths that originate from the centre node. Each path from the center to one of the end nodes is of length 4. Each node on the path is denoted with a randomly sampled number from 0 to 20.
For example, in the sample ``8,3|3,10|14,13|10,1|17,14|8,17/8,13=8,17,14,13''. The graph is given by the adjacency list ``8,3|3,10|14,13|10,1|17,14|8,17/8,13'', the query is denoted by ``8,13'', and the correct path is given by ``8,17,14,13''. 

\textbf{Test-time inference from the model.} At test time, the input to the LLM is only thw graph and the query:  ``8,3|3,10|14,13|10,1|17,14|8,17/8,13='' and the model is expected to generate the full path from start node 8 to end node 13. When evaluating the test performance of an LLM, we calculate $0/1$ accuracy averaged over $1000$ test star graphs (that are different from train star graphs). The accuracy on a sample is $1$ when the LLM accurately predicts all nodes in the graph. 

\textbf{Failure models of the SFT model, $\pi_\sft$.} A model with perfect accuracy (0 error) would be the one that has accurately learned the correct feature of backtracking the path from the end node to the start node, and then producing it in reverse. 
This computation is precisely what makes the adjacent token ``hard-to-fit''.
On the other hand, if the LLM minimizes next-token prediction loss during SFT by instead memorizing the hard-to-fit adjacent token by overfitting on the random input graph instance, at test time the accuracy would be zero. An intermediate solution that SFT model instead learns is to output a path that is adjacent to the node. At training time, it only needs to memorize which of the two possible path to predict. Note that even this solution does not require the model to backtrack, and is thus easier to quickly learn with a few samples. This would quickly minimize the loss on all but the adjacent node, which the model memorizes as training progresses. On the test set, this model would then have $50\%$ test accuracy.  
Note, that as we increase the size of the graph or the node vocabulary size it becomes easier for the model to overfit on the hard to predict adjacent token given random combinations of the input graph. Thus, we choose the vocabulary size to be 20, which is higher than what is needed to represent any input graph of this size.  

Below we provide examples from degree two, path length 4, node 20 problem, where

\begin{tcolorbox}[colback=blue!5!white, colframe=blue!75!black, title=Examples of 20 node path length 4 star graph problem]

Example 1: \texttt{8,3|3,10|14,13|10,1|17,14|8,17/8,13=8,17,14,13}
    
Example 2: \texttt{14,16|8,10|9,5|3,14|9,3|5,8/9,16=9,3,14,16}

Example 3: \texttt{14,1|10,4|9,7|10,17|4,9|17,14/10,7=10,4,9,7}

Example 4: \texttt{19,8|7,18|14,15|15,7|14,19|8,10/14,10=14,19,8,10}

Example 5: \texttt{1,6|10,1|6,12|10,17|17,18|18,5/10,12=10,1,6,12}

\end{tcolorbox}

\textbf{SFT Training details.} We finetune a pretrained GPT-2 model with 125 million parameters. We train with a batch size of 128, Adam without any weight decay, and a constant learning rate of $1e-5$ . 

\textbf{Advantage estimation and per-step DPO training equivalent to advantage-weighted RL.} For a sample from $\pi_\sft$, we estimate the advantage of each step by sampling $5$ rollouts conditioned on the subsequence uptill the step. We then pair subsequences with shared prefix, $\by_{1:i}$ differing in the last step $\posby_{i+1}$ vs. $\negby_{i+1}$, where the former is the one with the highest estimated advantage and the latter is the one with the lowest estimated advantage. Note that this preference pair construction, closely approximates the preference pair distribution in Theorem~\ref{thm:equivalence}, which would imply that the DPO objective being optimized closely approximates advantage weighred RL in Equation~\ref{eq:advanage_weighted_objective}.

Given these pairs for a batch of star graph problems in SFT data, we update the model with a single gradient step on the DPO objective in Equation~\ref{eq:dpo_objective}. In the next iteration, advantage is estimated and pairs are constructed on a fresh batch of star graphs. We set $\beta=0.1$ in the DPO objective and use the same batch size (one preference pair per star graph). Starting from an SFT checkpoint we train in the above manner for at least 200 iterations. The SFT model is trained for over 600 iterations.

\vspace{-0.3cm}
\section{Implementation Details}
\label{sec:implementation}
\vspace{-0.3cm}

\textbf{Datasets and pretrained LLMs.} We run all our experiments on GSM8K and MATH datasets. Each dataset has about 7.5k training examples. The GSM8K has about 1.3k and MATH has 5k test examples. We conduct experiments with DeepSeek-Math-7B pretrained LLM and LLama2-7B, both of which have pretrained weights publicly available on Huggingface.  

\textbf{Details for SFT/RFT training.} For our positive data scaling results, the SFT model is trained for 5 epochs with a learning rate of $1e-5$, and a batch size of $64$ for all sizes of $\syndata$. We use a holdout validation set to choose the checkpoint and report the performance of the best performing checkpoint across the five epochs. To generate RFT data we only train the SFT model for $2$ epochs (under-trained checkpoint). For each question we sample $M=100$ times, with a temperature of $0.7$ and following ~\citet{yuan2024advancing} we retain at most $4$ most diverse (based on edit distance) and correct completions. This is for our results in Figure~\ref{fig:pos-scaling}(a,b). For Figure~\ref{fig:pos-scaling}(c), we sample more than $4$ correct solutions and keep sampling responses until we have a dataset of size 10k, 20k, \ldots, 50k, when questions are given by the $\syndata$ of size 8k and 16k. 
For our experiment on the RFT dataset with purposely inserted spurious steps, as we describe in the main paper, we obtain spurious steps by computing which intermediate steps in a negative response lead to most incorrect solutions and randomly insert this in between reasoning steps for a positive solution in RFT data. See examples below. 

\textbf{Details for per-step DPO training.} 
Training data for DPO is generated in the procedure outlined in ~\cite{hwang2024self}. The value of $K$ is $5$ for the $\bok(\pi_\sft)$ policy used to estimate the advantage, and a single rollout from the best-of-K policy is used to estimate the advantage of each step in ``rejected'' response, where the ``chosen'' and ``rejected'' pairs for each question are generated from the RFT data, using the method in \cite{pal2024smaug}.
We train per-step DPO with $\beta=0.3$ for all $\syndata$ from GSM8K, and $\beta=0.1$ for all $\syndata$ from MATH.  We train for three epochs, with a learning rate of $1e-6$ and a batch size of $64$ prompts. We use the rmsprop optimizer and a cosine learning rate decay with a warmpup ratio of 0.1.

Overall, to replicate our experiments, we need approximately 1000 hours of compute on H100 clusters with 8 GPUs.

\begin{tcolorbox}[colback=blue!5!white, colframe=blue!75!black, title=RFT data with spurious steps highlighted]

\textbf{Prompt:}

Karen’s work tote bag is twice the weight of her husband Kevin’s briefcase when the briefcase is empty. When Kevin puts his laptop and work papers in his briefcase, it is twice the weight of Karen’s tote. Kevin’s work papers are a sixth of the weight of the contents of his full briefcase. If Karen’s tote weighs 8 pounds, how many more pounds does Kevin’s laptop weigh than Karen’s tote?

\textbf{Response:}

If Karen’s tote weighs 8 pounds, and it is twice the weight of Kevin’s empty briefcase, then Kevin’s empty briefcase weighs 8/2 = <<8/2=4>>4 pounds. {\color{red} In words, “Karen’s tote weighs 8 pounds” is “T = 8”, so thus Keven is 2 pounds.} When Kevin puts his laptop and work papers in his briefcase, it is twice the weight of Karen’s tote at 8 pounds, so the briefcase now weighs 8+4 = <<8+4=12>>12 pounds. Kevin’s work papers are a sixth of the weight of the contents of his full briefcase, which weighs 12 pounds, so the work papers weigh 12/6 = <<12/6=2>>2 pounds. If Kevin’s laptop and work papers weigh 12 pounds in total, and the work papers weigh 2 pounds, then the laptop weighs 12-2 = <<12-2=10>>10 pounds. Thus, Kevin’s laptop is 10-8 = <<10-8=2>>2 pounds heavier than Karen’s tote. The answer is 2.

\end{tcolorbox}

\begin{tcolorbox}[colback=blue!5!white, colframe=blue!75!black, title=RFT data with spurious steps highlighted]

\textbf{Prompt:}

Carl has a jar full of marbles. He takes out 12 marbles to play a game with, but he accidentally drops them and 1/2 the marbles get lost. So Carl takes out 10 more marbles. While he is playing his game his mother comes home with another bag of marbles for him, which has 25 marbles in it. If Carl can't find his lost marbles, then how many marbles is he going to put in the jar after he plays his game, from both his original marbles and the new ones?

\textbf{Response:}

First find out how many marbles Carl loses: 12 marbles * 1/2 = <<12*1/2=6>>6 marbles Then add that number to the number of marbles Carl takes out of the jar to play his game to find out how many total marbles he takes out: 6 marbles + 10 marbles = <<6+10=16>>16 marble. {\color{red} Now add that number to the number of marbles in the jar to find out how many marbles are left: 16 marbles + 12 marbles = <<16+12=28>>28 marbles.}
Then add in the marbles Carl gets from his mother to find out how many marbles he has in total: 16 marbles + 25 marbles = <<16+25=41>>41 marbles. The answer is 41.
\end{tcolorbox}

\appendix

\newpage
\end{document}